%% file: main.tex
\documentclass[runningheads]{llncs}
\usepackage{eccv}
\usepackage[normalem]{ulem}
\useunder{\uline}{\ul}{}

 \newcommand{\sj}[1]{{\color{red}{[SJ:#1]}}}
\usepackage{eccvabbrv}
\usepackage{graphicx}
\usepackage{booktabs}
\usepackage[accsupp]{axessibility} 
\usepackage{multirow}

\usepackage{hyperref}

\usepackage{orcidlink}

\usepackage{booktabs}  
\usepackage{multirow}  
\usepackage{graphicx}  

\usepackage[table,xcdraw]{xcolor}
\usepackage[normalem]{ulem}

\usepackage{float}
\usepackage{algorithm}
\usepackage{algorithmic}
\definecolor{codegreen}{rgb}{0.25, 0.5, 0.35}    
\begin{document}

\makeatletter
\newcommand\whline{\noalign{\ifnum0=`}\fi\hrule \@height 1.25pt \futurelet
	\reserved@a\@xhline}

\title{Progression as Latent Drift: Generative Forecasting of Slow-Evolving Pathologies \\[-0.3cm]}

\titlerunning{Progression as Latent Drift}

\author{Yuxiang Feng\inst{1,3}$^{\star\ddagger}$ \and
Juncheng Wang\inst{2}$^{\star}$ \and
Chao Xu\inst{3,4} \and
Wenlong Hou\inst{2} \\
Huihan Wang\inst{2} \and
Yijie Qian\inst{1,3} \and
Yang Liu\inst{3,4} \and
Baigui Sun\inst{3,4} \\
Yong Liu\inst{1}$^{\dagger}$ \and
Shujun Wang\inst{2}$^{\dagger}$ \\[-0.2cm]}

\authorrunning{Y.~Feng et al.}

\institute{Zhejiang University, Hangzhou, China \and
The Hong Kong Polytechnic University, Hong Kong, China \and
IROOTECH TECHNOLOGY, China \and
Wolf 1069 b Lab, Sany Group, China \\[0.2em]
\textnormal{$^{\star}$~Equal contribution.\quad
$^{\dagger}$~Corresponding authors.\\[0.3em]
\{fengyx@zju.edu.cn, wjc2830@gmail.com\}$^{\star}$,\quad
\{yongliu@iipc.zju.edu.cn, shu-jun.wang@polyu.edu.hk\}$^{\dagger}$}}

\maketitle
\begingroup\renewcommand\thefootnote{}\footnotetext{$^{\ddagger}$~This work was conducted in collaboration with IROOTECH TECHNOLOGY.}\addtocounter{footnote}{-1}\endgroup
\vspace{-0.8cm}

\begin{abstract}
Forecasting the future anatomy of slow-evolving neurodegenerative diseases could enable earlier, more targeted intervention and improve clinical trial design, but it remains challenging because true progression signals are subtle in longitudinal MRI. In this low-signal regime, transferring modern generative sequence models directly is unreliable: training is dominated by stable baseline anatomy and confounded by dense, sample-specific nuisance variation. We first provide a theoretical analysis that explains these failures through two modes. Identity collapse occurs when optimization is driven toward reproducing the current anatomy, which prevents the model from learning faint temporal change. The continuous interpolation trap arises when standard smooth networks cannot separate localized biological drift from pervasive noise, which leads to spurious changes that diffuse across the volume. To address both issues, we propose Latent Drift, a progressive generative framework that learns change in a compressed semantic representation rather than synthesizing full-resolution anatomy. This design removes pixel-level identity from the prediction target and concentrates model capacity on progression-relevant dynamics. We further apply Finite Scalar Quantization to the learned change representation, which suppresses small, high-frequency nuisance fluctuations while preserving consistent structural drift. Experiments on longitudinal 3D brain MRI show that Latent Drift improves patient-specific neuro-forecasting over diffusion and autoregressive transformer baselines across generative fidelity and clinically relevant evaluation metrics. Project page: \href{https://cutepkq.github.io/latent-drift}{https://cutepkq.github.io/latent-drift}.

  \keywords{Slow-Evolving Pathologies \and Generative Model \and Brain Simulation}
\end{abstract}

\section{Introduction}
\label{sec:intro}
\input{Sections/Introduction_v2}

\section{Preliminaries}
\input{Sections/Preliminary}

\section{The Pathologies of Decomposed Generative Forecasting}
\input{Sections/Findings}

\section{Our Approach: Progression as Latent Drift}
\input{Sections/Method}

\section{Experiments}
\input{Sections/Experiments}

\section{Related Work}
\input{Sections/Related_Work}

\section{Conclusion}
\input{Sections/Conclusion}

\section*{Acknowledgements}
This work was supported by the National Key Research and Development Program of China (Grant No. 2025YFF0511302). This work was partially supported by RGC Collaborative Research Fund (No. C5055-24G), the Start-up Fund of The Hong Kong Polytechnic University (No. P0045999), the Seed Fund of the Research Institute for Smart Ageing (No. P0050946), and Tsinghua-PolyU Joint Research Initiative Fund (No. P0056509), and PolyU UGC funding (No. P0053716).

\clearpage
\bibliographystyle{splncs04}
\bibliography{main}

% Appendix
\clearpage
\appendix
\setcounter{page}{1}

\section{Proof to Theorems}
\input{Sections/proof_draft}

\section{Experiments Details}
\input{Sections/Experiments_Appd}

\section{Detailed Ablation Results}
\input{Sections/Ablation_Appd}

\section{More Gradient Visualization}
\input{Sections/Vis_Appd}

\end{document}

%% file: Sections/Introduction_v2.tex
Dynamic simulation of organ evolution~\cite{schmidt2007evolution,distefano2015dynamic,caldwell2003dynamic} is essential for forecasting disease trajectories, particularly for neurodegeneration and other slow-evolving brain pathologies that progress along largely irreversible courses. Because neuronal loss and downstream structural changes (e.g., cortical thinning and ventricular enlargement) cannot be recovered~\cite{Blinkouskaya2021BrainAM, Apostolova2012HippocampalAA}, substantial tissue damage may already be present by the time symptoms become clinically apparent~\cite{Bateman2012ClinicalAB, Dubois2016PreclinicalAD}. This delayed observability narrows the effective window for intervention and complicates the development and evaluation of disease-modifying therapies, which increasingly depend on identifying who will deteriorate, where, and how fast~\cite{Nakashima2025TherapeuticTW, Zhang2024DiseasemodifyingTF, Assuno2022MeaningfulBA}. 

\input{Figures/TexFig1}
Generative forecasting predicts a patient’s future anatomical state and simulates the personalized morphological trajectory of disease. Localized anticipation of progression can support earlier, more targeted care for conditions such as Alzheimer’s and Parkinson’s disease, where structural decline unfolds over years~\cite{Rohrer2012StructuralBI, Pereira2020LongitudinalDO}. It can also enrich clinical trials by prioritizing participants likely to show measurable progression, and reduce caregiving burden through more proactive management before irreversible damage accumulates~\cite{deVugt2013TheIO}.

\noindent
\begin{minipage}[c]{0.40\linewidth}
    \ \ Recent generative sequence models synthesize audio~\cite{MelQCD,haji2026taming,wang2025language,wang2026guided} and visual~\cite{cui2026lol,km2026phyeduvideo,brokman2026training,wan2025wan,ma2024follow,ma2025controllable} content well, which makes them a natural starting point for forecasting future brain anatomy~\cite{Rombach2021HighResolutionIS, Croitoru2022DiffusionMI, Kazerouni2022DiffusionMI, Peng2022GeneratingRB}, as shown by Fig.~\ref{fig:1}-(a).
\end{minipage}
\hfill
\begin{minipage}[c]{0.58\linewidth}
  \raggedright
  \vspace{0.8\baselineskip}
  \includegraphics[width=0.82\linewidth]{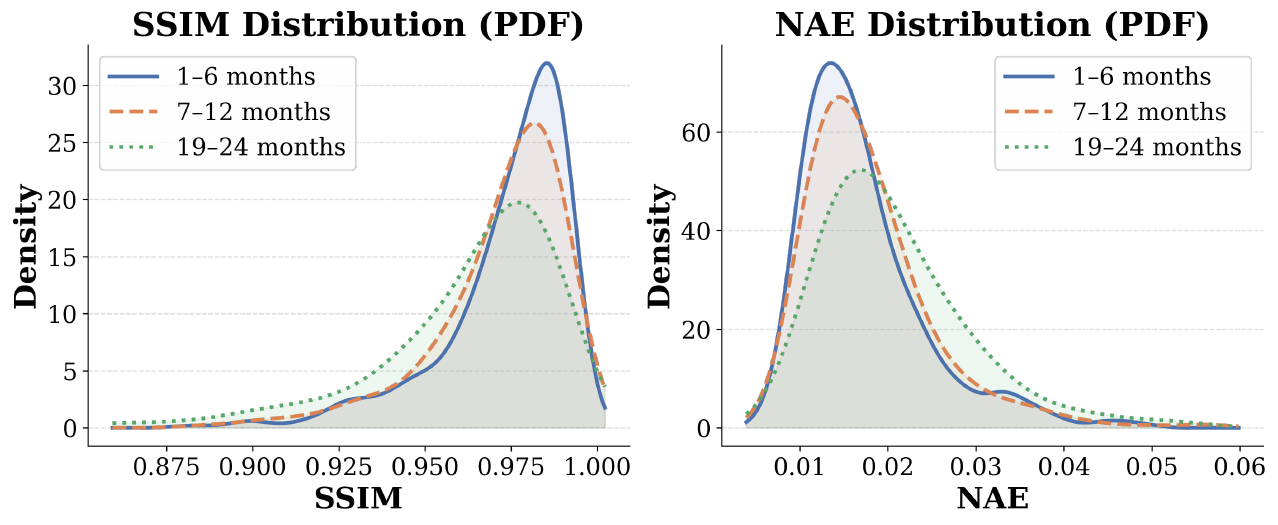}
    \captionof{figure}{Statistic similarity between current and future states.}
    \label{fig:current_future_similarity}
\end{minipage}%
\hfill

However, as shown in Fig.~\ref{fig:current_future_similarity}, forecasting neurodegeneration is different because current and future MRI states are statistically highly similar, making the temporal progression signal far weaker. Unlike natural videos with salient motion, neurodegeneration unfolds over long time scales~\cite{Henneman2009HippocampalAR, Leung2013CerebralAI}: across a typical one-year interval, disease-related morphological change may be under 1\% of the total volumetric variance in a 3D MRI scan~\cite{Narayanan2020BrainVL}. In this low-signal regime, naive transfer is unreliable, as training is dominated by static anatomy or by nuisance variation resembling scanner noise rather than the subtle biological progression of interest~\cite{Farki2025ForecastingFA, Wittens2021InterAI}.

To clarify this mismatch, our first contribution is a mathematical analysis showing that directly applying state-of-the-art generative sequence models to neurodegenerative forecasting is structurally prone to failure. The analysis yields two characteristic failure modes that expose intrinsic limitations of existing approaches. First, \textbf{identity collapse} occurs when optimization is dominated by the much larger baseline anatomy, pushing the model toward an approximate identity mapping that reproduces static neuroanatomy instead of learning faint temporal drift, as shown by Fig.~\ref{fig:1}-(b). Second, \textbf{the continuous interpolation trap} arises under dense, sample-specific nuisance variation. For Lipschitz-continuous networks, sparse semantic progression cannot be separated cleanly from pervasive noise. The learned predictor then behaves as a spurious interpolator that diffuses variance across the full volume instead of recovering localized neurodegenerative change~\cite{lipschiSca, lipschiFaz}. Therefore, \textit{How can we constrain a generative model to reliably predict subtle, localized atrophy trajectories, instead of memorizing baseline anatomy or amplifying scanner-dependent noise?}

In this paper, we address these two failure modes with a progressive generative framework designed to separate the disease-relevant progression signal from both dominant background anatomy and nuisance variability. \underline{First}, to mitigate the dominance of baseline anatomy and prevent identity collapse, we project each MRI scan into a compressed semantic latent space and train an autoregressive model to predict \textbf{the latent drift} between the current and future representations, rather than \textbf{reconstructing the entire future image}. By removing pixel-level identity as the modeling target, the model is encouraged to allocate its capacity to learning the temporal change that reflects disease progression, illustrated by Fig.~\ref{fig:1}-(c). \underline{Second}, to reduce the influence of nuisance variation and avoid spurious interpolation, we apply Finite Scalar Quantization (Fig.~\ref{fig:1}-(d)) to the latent change representation. This step acts as a \textbf{topological dead-zone filter}: small, high-frequency variations are mapped to zero, while consistent structural changes associated with disease progression stay above the quantization threshold and are preserved for forecasting.

We evaluate our Latent Drift framework on longitudinal 3D brain MRIs. By bypassing Identity Collapse and filtering nuisance noise, our approach improves neuro-forecasting over prior diffusion and autoregressive baselines. Our main contributions are as follows.
\pagebreak[4]

\begin{itemize}
\item \textbf{Task formalization and failure analysis:} We cast forecasting of slow-evolving pathologies as conditional generative synthesis and identify two fundamental optimization failures: \textbf{identity collapse} from dominant stationary anatomy and the \textbf{continuous interpolation trap} from dense nuisance variability.
\item \textbf{Latent Drift framework:} We introduce a progressive generative architecture that predicts progression in a compressed semantic space, disentangling biological change from baseline anatomy and amplifying the sparse pathological signal.
\item \textbf{FSQ mechanism:} We analyze Finite Scalar Quantization as a \textbf{topological dead-zone filter} that suppresses high-frequency nuisance noise while preserving structurally consistent drift, mitigating the limitations of smooth predictors.
\item \textbf{Empirical performance:} We achieve state-of-the-art patient-specific 3D MRI forecasting, outperforming diffusion and autoregressive transformer baselines on both generative fidelity and clinically relevant metrics.
\end{itemize}

%% file: Figures/TexFig1.tex
\begin{figure}[t]
    \centering
    \includegraphics[width=\linewidth]{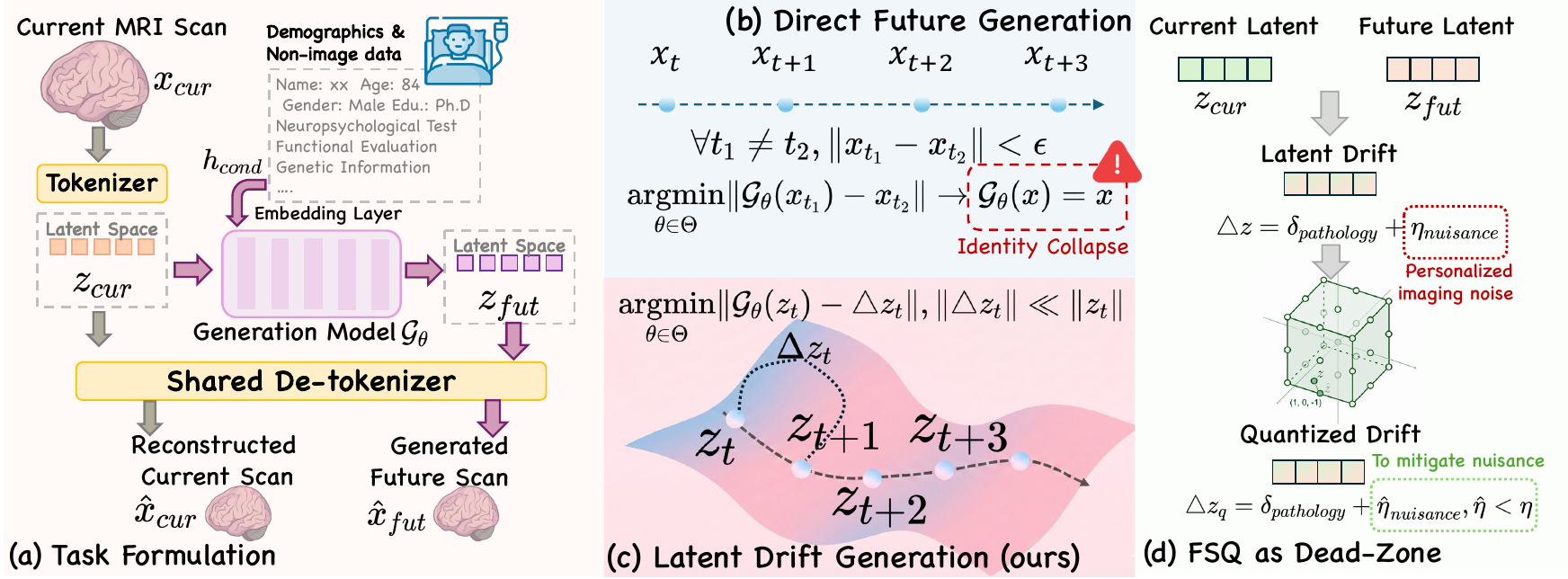}
    \caption[Progression as Latent Drift]{
        \textbf{Progression as Latent Drift.} 
        \textbf{(a) Task Formulation:} Autoregressive forecasting of future MRI states in a compressed semantic latent space. 
        \textbf{(b) Direct Future Generation:} Predicting the absolute future state induces \textit{Identity Collapse}, as stationary background anatomy overwhelms microscopic biological changes. 
        \textbf{(c) Latent Drift Generation (Ours):} Predicting the temporal residual (\texorpdfstring{$\Delta z$}{Delta z}) bypasses this collapse, isolating the true pathological trajectory. 
        \textbf{(d) FSQ as a Topological Dead-Zone:} To break the \textit{Continuous Interpolation Trap} caused by dense imaging noise, Finite Scalar Quantization (FSQ) serves as a non-Lipschitz filter that annihilates nuisance variations while preserving sparse semantic drift.
    }
    \label{fig:1}
\end{figure}

%% file: Sections/Preliminary.tex
\subsection{Problem Formulation}
Let $\mathcal{D} = \{(x_{cur}^{(i)}, x_{fut}^{(i)}, c^{(i)}, \Delta t^{(i)})\}_{i=1}^N$ be a dataset of longitudinal 3D brain MRIs. Given a baseline scan $x_{cur} \in \mathbb{R}^{H \times W \times D}$, patient clinical metadata $c$, and a time horizon $\Delta t$, the objective is to learn a conditional generative model $p_\theta(x_{fut} \mid x_{cur}, c, \Delta t)$ to forecast the future anatomical state $x_{fut}$.

\subsection{Baseline: Latent Generative Forecasting}
Due to the computational intractability of modeling high-resolution 3D volumes directly in pixel space, the standard baseline paradigm employs a three-stage latent generative framework.
\begin{itemize}
    \item \textbf{Spatial Compression}: A pre-trained encoder $\mathcal{E}$ maps the input volumes into a compact, lower-dimensional latent space: $z_{cur} = \mathcal{E}(x_{cur})$ and $z_{fut} = \mathcal{E}(x_{fut})$.
    \item \textbf{Latent Generation}: The clinical variables and time horizon are projected into dense conditioning embeddings, $h_{cond} = \text{MLP}(c, \Delta t)$. A sequence model $\mathcal{G}_\theta$ (e.g., an autoregressive transformer) is then trained to directly synthesize the future latent state:
    \begin{equation}
        \hat{z}_{fut} = \mathcal{G}_\theta(z_{cur}, h_{cond}).
    \end{equation}
    \item \textbf{Reconstruction}: The forecasted MRI is decoded back to the pixel domain via a pre-trained spatial decoder $\mathcal{D}$, yielding $\hat{x}_{fut} = \mathcal{D}(\hat{z}_{fut})$.
\end{itemize}

%% file: Sections/Findings.tex
Having established the standard latent generative baseline, we analyze the optimization dynamics of forecasting slow-evolving temporal systems. To understand why standard architectures fail on Alzheimer's progression, we decompose the future latent state into three signals:
\begin{equation}
    z_{fut} = z_{cur} + \delta_{pathology} + \eta_{nuisance},
\end{equation}
where $z_{cur}$ is the dominant stationary anatomy, $\delta_{pathology}$ is the sparse biological progression, and $\eta_{nuisance}$ is the dense, sample-personalized imaging noise. In this section, we show how the properties of these terms induce optimization pathologies that make standard continuous generation insufficient.

\subsection{Source Dominance and the Identity Trap}
Our first hypothesis addresses the relative magnitudes of the decomposed terms: because the structural change over a typical clinical horizon is microscopic, 

\noindent
\begin{minipage}[t]{0.49\linewidth}
\vspace{0pt}
    the magnitude of the stationary background vastly exceeds the biological signal ($||z_{cur}|| \gg ||\delta_{pathology}||$). 
    We hypothesize that when a generative sequence model $\mathcal{G}_\theta$ is optimized to directly synthesize the absolute future state ($z_{fut}$), this magnitude imbalance causes the stationary background to overwhelmingly dominate the training 
\end{minipage}
\hfill
\begin{minipage}[t]{0.49\linewidth}
\vspace{0pt}
\centering
\includegraphics[width=\linewidth]{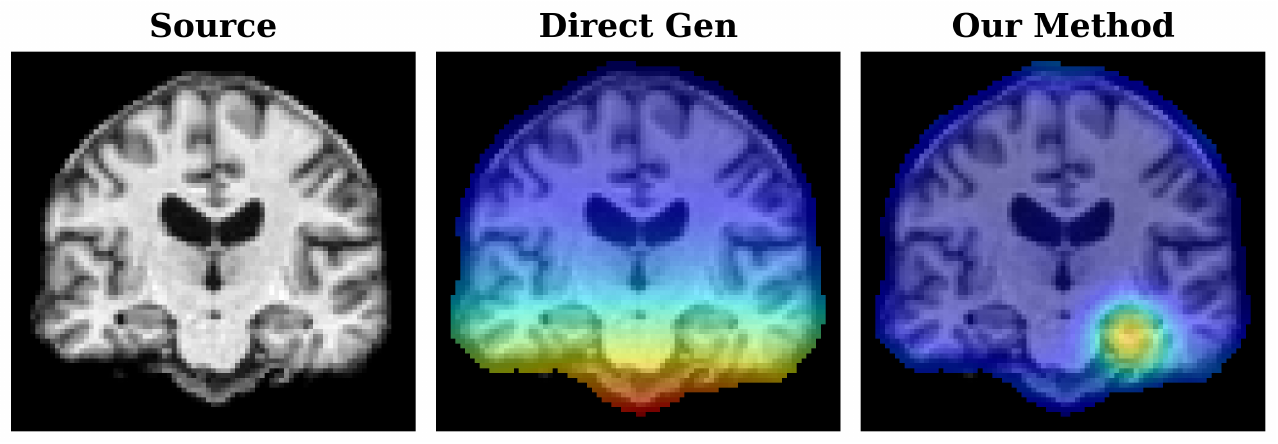}
\captionof{figure}{Visualization of gradient-based signal differences between current and predicted future states.}
\label{fig:gradient_vis}
\end{minipage}

\noindent gradients, trapping the model in an identity function ($f(x) \approx x$).

\paragraph{Empirical Validation:} To empirically validate this source dominance, we trained a baseline continuous sequence model to directly predict $z_{fut}$ conditioned on patient metadata and time horizon. To trace the optimization focus, we visualized the spatial gradients of the objective function with respect to the generated future state during training.

As shown in Fig.~\ref{fig:gradient_vis}, the gradient analysis shows a clear difference in how capacity is allocated. When generating $z_{fut}$ directly, the training gradients spread across the spatial volume, indicating that the network spent its capacity minimizing reconstruction error on static anatomy such as the skull and unaffected hemispheres, and largely ignored the localized progression. When the target was shifted to the temporal difference ($\Delta z$), the gradients localized around regions of true change such as the hippocampus. This confirms that predicting the absolute state induces Identity Collapse, and that a successful architecture must remove $z_{cur}$ from the objective to isolate the progression content. Analogous spatial-temporal decoupling has proven effective for motion modeling in natural video~\cite{ma2025follow}, though our setting demands isolating a far sparser biological signal.

\subsection{The Continuous Interpolation Trap}
Shifting the predictive target to the temporal residual ($\Delta z = \delta_{pathology} + \eta_{nuisance}$) bypasses Identity Collapse, but it exposes the model to a second pathology. The dense, sample-personalized nature of $\eta_{nuisance}$ perturbs the convergence of continuous neural networks and prevents recovery of the true biological trajectory.

We prove this using statistical learning theory. The biological drift $\delta$ is anatomically sparse, supported on a localized, low-dimensional manifold. The imaging noise $\eta$ is a high-frequency perturbation distributed across the entire latent dimension. When a continuous sequence model fits this composite signal, it behaves as a spurious interpolator.

We formalize this failure mode in the following theorem, demonstrating that continuous, Lipschitz-bounded models are mathematically incapable of isolating sparse semantic drift when corrupted by dense, sample-personalized noise.

\begin{theorem}[Continuous Interpolation Trap(Informal, see Appendix for formal version)]
\label{thm:trap}
Let $\Delta z(x)=\delta(x)+\eta(x)$, where $\delta(x)\in\mathbb R^d$ is sparse 
($\|\delta(x)\|_0 \le s \ll d$) and $\eta(x)$ is dense, zero-mean nuisance noise 
with $\|\eta(x)\|_\infty \le \varepsilon$ and non-zero entries in all coordinates.

Let $f_\theta$ be an $L$-Lipschitz model trained by ERM to fit $\Delta z$.

Then, even if $f_\theta$ interpolates the training data,

\begin{equation}
    \mathbb E\|f_\theta(x)-\delta(x)\|_2^2
\;\gtrsim\;
\mathbb E\|\eta(x)\|_2^2,
\end{equation}

and for unseen $x$, $f_\theta(x)$ is generically dense.  

Hence, no Lipschitz-continuous predictor trained on $\Delta z$
can recover the sparse support of $\delta$.
\end{theorem}

\paragraph{Implications for Forecasting:} Theorem 1 establishes a mathematical limit on continuous generative forecasting. Because modern neural networks are Lipschitz-continuous, minimizing expected error over the dense $\eta_{nuisance}$ distribution forces the model to spread small, continuous variations across the entire predicted latent volume. The resulting generation smears spurious variance across the volume, which compromises the structural fidelity of the predicted Alzheimer's trajectory.

\textbf{In summary}, the empirical trap of $z_{cur}$ and the theoretical trap of $\eta_{nuisance}$ together impose an architectural requirement: \emph{the generative framework must target the temporal drift to escape Identity Collapse, and it must have a non-Lipschitz, non-linear thresholding property to break the continuous interpolation trap.}

%% file: Sections/Method.tex
\input{Figures/TexFigMain}
Our pathology analysis imposes two constraints at once: the model must escape source dominance while also breaking the continuous interpolation trap. To meet both, we introduce a two-stage generative framework, Progression Latent Drift, that separates the semantic biological trajectory from the stationary background and the dense imaging noise. The overall pipeline is shown in Fig.~\ref{fig:framework}.

\textbf{Architecture overview.} The framework combines a residual tokenization stage with an autoregressive sequence model. Rather than mapping $x_{cur} \rightarrow x_{fut}$ directly, we train a Latent Drift Tokenizer to isolate and discretize the temporal difference between scans. We then train a Decoder-Transformer to forecast this discrete progression signal, conditioned on the patient's baseline state and clinical metadata.

\subsection{The Latent Drift Tokenizer (Escaping Identity Collapse)}

To address the Identity Collapse caused by the large spatial overlap of $x_{cur}$, we use an architectural bottleneck that prevents the network from reconstructing the stationary anatomy from the temporal target.

Given a longitudinal pair $(x_{cur}, x_{fut})$, a tokenizer $\mathcal{E}$ first compresses the volumes into dense continuous representations: $z_{cur} = \mathcal{E}(x_{cur})$ and $z_{fut} = \mathcal{E}(x_{fut})$. Instead of quantizing these absolute states independently, we compute the continuous latent drift:
\begin{equation}
    \Delta z_{raw} = z_{fut} - z_{cur}.
\end{equation}

To sanitize this signal of dense, sample-personalized noise ($\eta_{nuisance}$) and enable discrete sequence modeling, we project the continuous drift through a Finite Scalar Quantization (FSQ) bottleneck. We define a uniform scalar grid bounded by $L_{max}$ with a step size $s$. Utilizing the stop-gradient operator, $\text{sg}(\cdot)$, we formulate the complete FSQ forward pass and straight-through estimator (STE) backward pass as a single operation:
\begin{equation}
    \Delta z_q = \Delta z_{raw} + \text{sg}\left( s \cdot \text{clip}\left( \lfloor \frac{\Delta z_{raw}}{s} \rceil, -L_{max}, L_{max} \right) - \Delta z_{raw} \right).
\end{equation}
The rounding operation $\lfloor \cdot \rceil$ acts as a topological dead-zone: any continuous noise fluctuation where $|\Delta z_{raw}| < s/2$ is mapped exactly to zero.

Finally, to reconstruct the future state, the quantized drift is added back to the baseline continuous latent and passed through the de-tokenizer $\mathcal{D}$:
\begin{equation}
    \hat{x}_{fut} = \mathcal{D}(z_{cur} + \Delta z_q).
\end{equation}

This residual objective forces the autoencoder to allocate its discrete codebook capacity to the semantic structural changes ($\Delta z_q$), keeping the target separate from the stationary background.

\subsection{Autoregressive Drift Generation}
With the Latent Drift Tokenizer trained, we extract the discrete drift tokens $\Delta z_q$ for all longitudinal pairs in the training corpus. The forecasting task then becomes a discrete sequence modeling problem over the isolated biological progression~\cite{qian2025think}.

We use a Decoder-Transformer, $\mathcal{G}_\theta$, to autoregressively synthesize the quantized drift. The model is conditioned on the continuous baseline anatomy $z_{cur}$ and a fused embedding of the clinical metadata $c$ and time horizon $\Delta t$:
\begin{equation}
    P(\Delta \hat{z}_q \mid x_{cur}, c, \Delta t) = \mathcal{G}_\theta(\Delta \hat{z}_q \mid z_{cur}, \text{MLP}(c, \Delta t)).
\end{equation}

By optimizing the cross-entropy loss over the FSQ vocabulary, the Transformer models the probability distribution of semantic changes. During inference, the predicted drift tokens are decoded and added to the patient's baseline scan, yielding a patient-specific forecasted MRI.

\subsection{Theoretical Resolution: FSQ as a Topological Dead-Zone Filter}

Section~3 established a structural limitation: continuous Lipschitz predictors trained on
\(
\Delta z_{\mathrm{raw}} = \delta + \eta
\)
must interpolate the dense nuisance component $\eta$, thereby destroying the sparsity of the true biological drift $\delta$.

We now show that applying Finite Scalar Quantization (FSQ) to the latent residual introduces the precise non-Lipschitz behavior required to break this failure mode.

\paragraph{Dead-Zone Quantization.}

Let $\delta \in \mathbb{R}^d$ be $s$-sparse with minimum signal magnitude $\gamma > 0$, and let $\eta$ be dense nuisance noise satisfying $\|\eta\|_\infty \le \varepsilon$ with $\varepsilon < \gamma$. The separation condition $\varepsilon < \gamma$ corresponds to a minimal signal strength ($\beta$-min) assumption commonly imposed in exact support recovery analyses under noise \cite{betamin}. Without such a margin condition, sparse support is statistically non-identifiable regardless of the estimator.

Define a scalar quantizer with step size $h$:
\begin{equation}
Q_h(z)_j = h \cdot \left\lfloor \frac{z_j}{h} \right\rceil.
\end{equation}

We calibrate the quantization scale such that
\begin{equation}
h > 2\varepsilon.
\end{equation}

\begin{theorem}[FSQ as a Topological Dead-Zone Filter (Informal, see Appendix for formal version)]
Under the conditions above, quantization of the residual
\begin{equation}
    \Delta z_q = Q_h(\delta + \eta)
\end{equation}
has the following properties:

1. All nuisance-only coordinates vanish:
    \begin{equation}
        j \notin \mathrm{supp}(\delta)
    \Rightarrow
    \Delta z_{q,j} = 0.
    \end{equation}

2. The support of the biological drift is preserved:
    \begin{equation}
        \mathrm{supp}(\Delta z_q) = \mathrm{supp}(\delta).
    \end{equation}

3. The dense noise manifold collapses to a sparse discrete set prior to learning.

Consequently, the generative model no longer receives dense stochastic supervision, and the Continuous Interpolation Trap no longer applies.
\end{theorem}

\paragraph{Intuitive Remark}

For coordinates outside the true pathology support, the residual equals pure noise with magnitude strictly below $h/2$, forcing quantization to zero.
On the true support, the signal magnitude exceeds the dead-zone threshold, ensuring non-zero discrete activation.
Since $Q_h$ is discontinuous at the bin boundaries, it is non-Lipschitz, which breaks the smooth interpolation behavior of standard neural predictors.

\paragraph{Implication.}

FSQ is more than a compression mechanism. Applied to the latent drift, it acts as a structural projection that eliminates dense nuisance noise, restores sparse support geometry, and converts continuous regression into discrete event modeling. By applying this topological dead-zone before autoregressive modeling, Quantized Latent Drift addresses the Continuous Interpolation Trap and supports structurally faithful forecasting.

%% file: Figures/TexFigMain.tex
\begin{figure}[t]
    \centering
    \includegraphics[width=\linewidth]{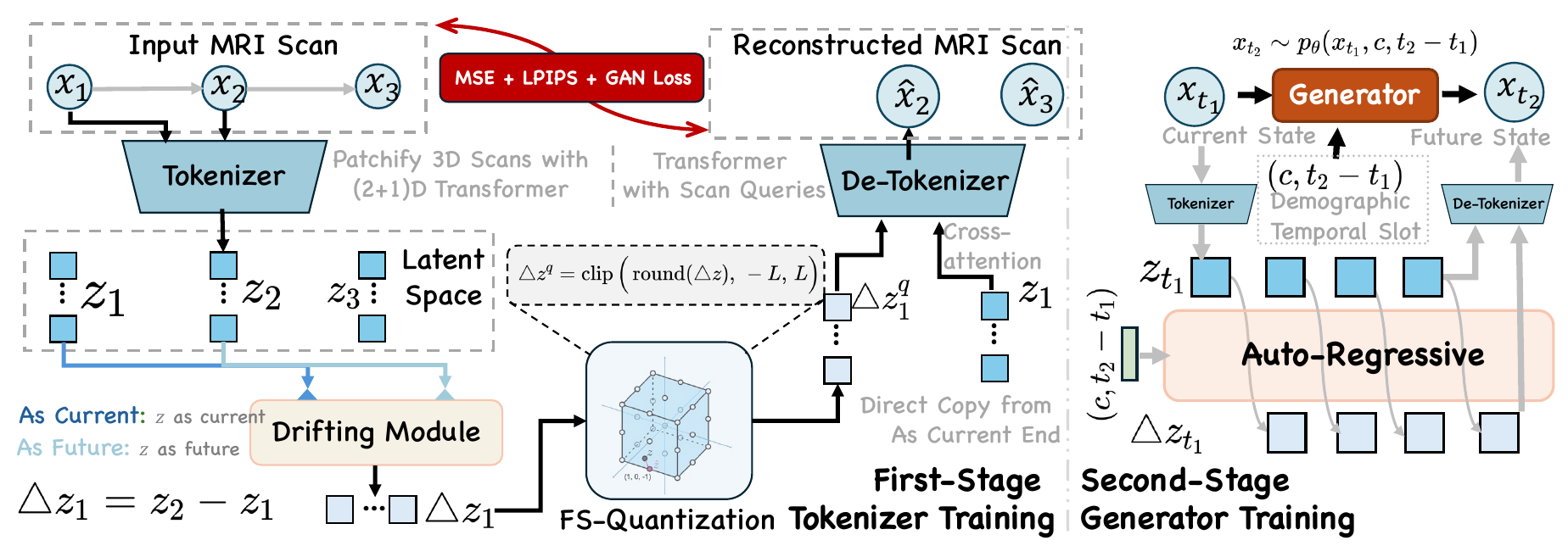}
    \caption{
        \textbf{The Progression Latent Drift Framework.} Our two-stage architecture explicitly disentangles sparse biological progression from stationary anatomy and dense imaging noise. 
        \textbf{First-Stage (Latent Drift Tokenizer):} Longitudinal MRI scans are patchified and compressed into continuous latent representations ($z$). To escape Identity Collapse, a Drifting Module computes the temporal residual ($\Delta z_1 = z_2 - z_1$). Finite Scalar Quantization (FSQ) discretizes this continuous drift, acting as a topological dead-zone filter to strictly suppress non-semantic nuisance noise. The future scan is then reconstructed by fusing the quantized drift ($\Delta z^q_1$) with the baseline latent ($z_1$) via cross-attention. 
        \textbf{Second-Stage (Generator Training):} With the biological trajectory successfully isolated, an autoregressive sequence model learns to forecast the discrete drift tokens, conditionally guided by the baseline anatomy ($z_{t_1}$), patient demographics, and the target temporal horizon ($t_2 - t_1$).
    }
    \label{fig:framework}
\end{figure}

%% file: Sections/Experiments.tex
\subsection{Evaluation Setup}
\noindent \textbf{Benchmarks.} We utilize data from the ADNI \cite{jack2008alzhe} and AIBL \cite{ellis2009australian} datasets. During preprocessing, all 3D MRI volumes are downsampled to a resolution of 93 × 112 × 96. By leveraging the patients' longitudinal follow-up information, we construct a total of 3,981 current-future scan pairs. From this cohort, 1,140 pairs are allocated to the testing set.

\noindent \textbf{Evaluation metrics.} We evaluate the generated 3D MRIs across three dimensions: Structural Similarity, which assesses spatial agreement between generated and ground truth volumes via SSIM and NCC metrics; Downstream Clinical Utility, which measures the diagnostic value of synthesized images using a pre-trained Alzheimer's Disease classification model; and Generative Fidelity, which evaluates distribution-level realism using Fréchet Inception Distance (FID) with features extracted by I3D and our domain-specific classifier (see Appendix for details).

\noindent \textbf{Clinical evaluation protocol.} For Downstream Clinical Utility, we use a frozen AD classifier (a ViViT-style model with over 91\% accuracy on real scans) that is run on each generated future scan, and the predicted label is compared against the patient's real ground-truth diagnosis at the matching time. The classifier and forecasting cohorts are patient-disjoint, so no patient seen during classifier training appears in the forecasting test set, and we draw a single forecast per (patient, target time) without multi-sample averaging or ensembling.

\subsection{Main Results}
\input{Tables/Tab_Main}
The results in Table~\ref{tab:main} support our analysis of Identity Collapse and the Continuous Interpolation Trap. Latent Drift achieves the best \textbf{Diff-SSIM (0.8204)} and \textbf{NCC (0.9880)}, which indicates that targeting the latent temporal residual $\Delta z$ helps the model avoid being dominated by stationary anatomy. CycleGAN \cite{zhu2017unpaired} achieves higher generative fidelity in terms of \textbf{FID (12.13 and 7.03)}, but its lower \textbf{Diff-SSIM (0.7713)} suggests that its optimization mainly reproduces static neuroanatomy rather than capturing the sparse biological drift needed for forecasting. Wilcoxon signed-rank tests over patient-level Accuracy, $F_1$, and Diff-SSIM confirm that these gains are statistically significant, with our method outperforming every baseline at $p<0.05$.

The clinical metrics point to the role of the Finite Scalar Quantization (FSQ) step. Acting as a topological dead-zone filter, FSQ suppresses the dense nuisance noise that drives continuous models toward spurious variance, which helps keep the predicted trajectory structurally consistent. Our method leads on \textbf{$F_1$ (87.51)} and \textbf{accuracy (88.33)} over diffusion-based baselines such as Palette~\cite{saharia2022palette} and I2I-DiT~\cite{peebles2023scalable}. Overall, separating biological progression from baseline anatomy through a quantized residual is an effective strategy for patient-specific neuro-forecasting.

\noindent 
\begin{minipage}[t]{0.48\linewidth} 
\vspace{0pt} 
    
\paragraph{Longitudinal Trajectory Fidelity}
Figure \ref{fig:trajectory} shows the longitudinal trajectory analysis, which tests whether the framework can simulate patient-specific neurodegeneration over long clinical intervals. We track the structural similarity between the baseline scan and the forecasted states over a four-year period, and find that our approach follows the rate of anatomical decline with good accuracy and tracks the ground-truth progression closely.
\end{minipage}
\hfill 
\begin{minipage}[t]{0.5\linewidth}
\vspace{0pt}

\centering
\includegraphics[width=\linewidth]{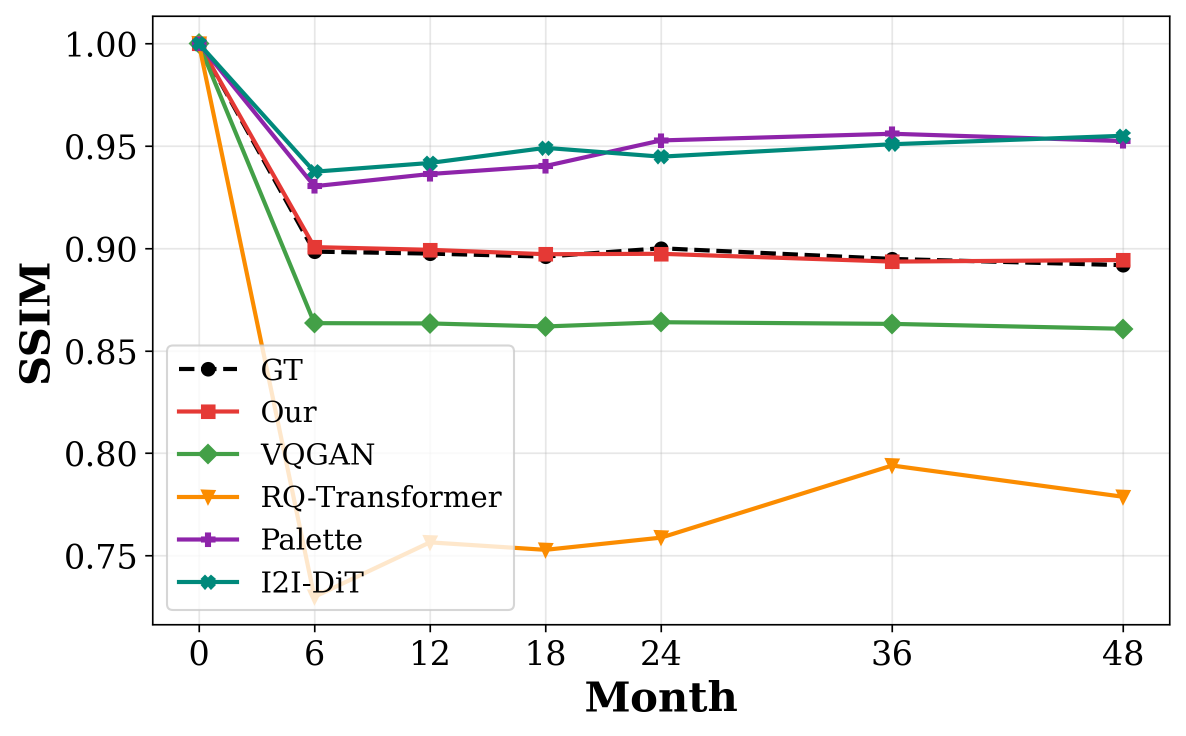}
\captionof{figure}{Personalized Longitudinal Progression Trajectory.}
\label{fig:trajectory}
\end{minipage}
\vspace{2mm}

\noindent This suggests the model avoids the Identity Collapse seen in competing methods. The high SSIM of the Palette and I2I-DiT trajectories indicates that these models settle into a near-identity mapping, mostly reproducing stationary anatomy instead of the sparse pathological signal $\delta$.

In contrast, the VQGAN~\cite{esser2021taming} and RQ-Transformer~\cite{lee2022autoregressive} baselines drift and fluctuate, showing that standard continuous and autoregressive models struggle to stay anatomically consistent under low-signal conditions. Our framework remains stable over the 48-month horizon, which is consistent with the role of the FSQ bottleneck in preventing the spread of spurious variance across the volume. By recovering the personalized atrophy trajectory, Latent Drift provides the structural consistency needed for clinical forecasting.

\paragraph{Consistency Across Fast- and Slow-Changing Regions}
A natural concern with a globally uniform FSQ grid is that anatomical regions with very different rates of change might be quantized inconsistently. Because FSQ is applied in the learned latent space rather than in pixel space, the encoder maps anatomically heterogeneous regions onto a comparable scale before the dead-zone is applied. Fig.~\ref{fig:roi} reports the recovered-to-ideal change ratio on three regions of interest that span roughly a $23\times$ range of dynamic change (hippocampus, cerebellum, and ventricle). Our method stays within $[0.93, 1.18]$ of the ideal ratio across all three regions, with a mean absolute deviation of $0.12$, which is $2.6$ to $5.8\times$ closer to the ideal than the competing baselines. This indicates that the shared grid does not systematically under-represent fast-changing regions in practice, though a region-aware quantization scheme remains an open direction.

\begin{figure}[!htbp]
    \centering
    \includegraphics[width=\linewidth]{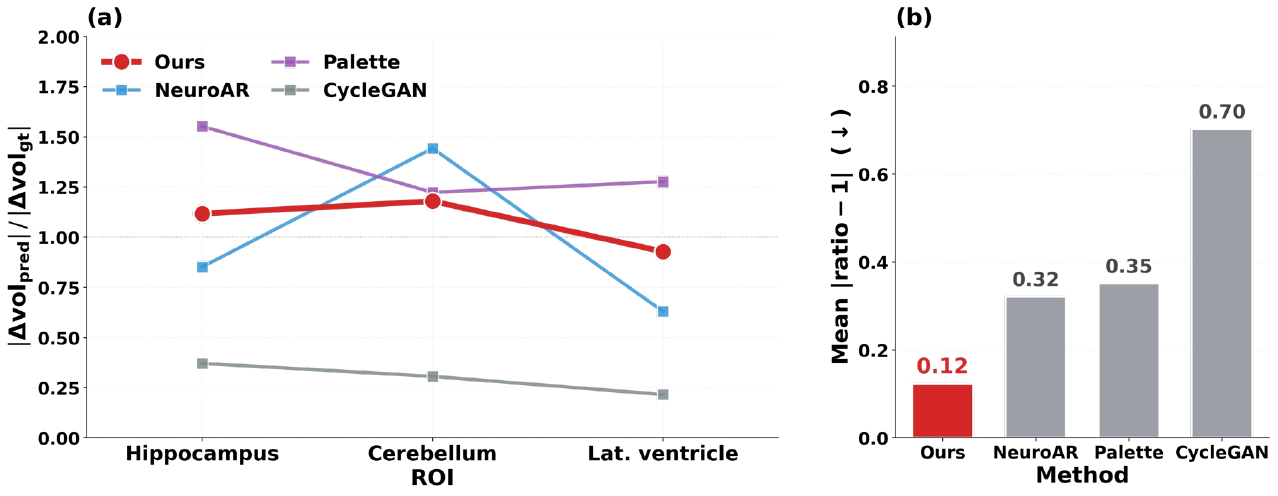}
    \caption{Recovered-to-ideal change ratio across three regions of interest spanning a wide range of dynamic change. Our method stays close to the ideal ratio of 1 across regions, indicating that the shared latent-space quantization grid handles anatomically heterogeneous regions consistently.}
    \label{fig:roi}
\end{figure}

\subsection{Ablation Studies}

\noindent 
\begin{minipage}[t]{0.50\linewidth} 
\vspace{0pt} 
    \paragraph{Which prediction target is better?}
    Table~\ref{tab:target} supports moving from pixel-level synthesis to latent drift prediction, which addresses the Identity Collapse from our analysis. When the model 
\end{minipage}
\hfill 
\begin{minipage}[t]{0.48\linewidth} 
\vspace{0pt}\vspace{\dimexpr-\baselineskip+0.3em\relax} 
\centering
    \captionof{table}{Ablation of Prediction Targets.}
    \label{tab:target}
    \vspace{1mm} 
    {\renewcommand{\arraystretch}{1.0}
    \resizebox{\textwidth}{!}{
        \begin{tabular}{c|c|c|c|c}
        \whline
        \textbf{Target} & FID  & Diff-SSIM & Acc.  & $F_1$  \\ \whline
        Pixel           & 44.38 & 0.7625    & 77.19 & 72.75 \\ \hline
        \rowcolor[HTML]{EFEFEF} 
        Drift (ours)    & 13.92 & 0.8205    & 88.33 & 87.52 \\ \whline
        \end{tabular}
    }}
\end{minipage}
\vspace{2mm}

\noindent performs direct pixel-level forecasting, it reaches a higher FID of 44.38 and a lower Diff-SSIM of 0.7625. This indicates that the optimization is dominated by reproducing stationary neuroanatomy rather than capturing faint biological progression. Shifting the target to the latent temporal residual $\Delta z$ concentrates model capacity on the progression-relevant dynamics.

\noindent 
\begin{minipage}[t]{0.48\linewidth}
\vspace{0pt}
    \paragraph{How does FSQ perform comparing with other quantization methods?}
  The ablation results in Table~\ref{tab:quan} support the choice of Finite Scalar Quantization as the topological dead-zone filter for neuro-forecasting. While Vector Quantization~(VQ) and Lookup-Free
\end{minipage}
\hfill 
\begin{minipage}[t]{0.48\linewidth} 
\vspace{0pt}
\vspace{-5mm}
\captionof{table}{Ablation of Quantization Mechanisms}
\centering
\renewcommand{\arraystretch}{1.0}{
\resizebox{\textwidth}{!}{
\begin{tabular}{c|c|c|c|c}
\whline
\textbf{Quantizers} & rFID  & Diff-SSIM & $Acc.$  & $F_1$  \\ \whline
\rowcolor[HTML]{EFEFEF} 
FSQ (ours)          & 13.77 & 0.8224    & 88.24 & 87.41 \\ \hline
VectorQuan.         & 14.04 & 0.8249    & 86.67 & 85.52 \\ \hline
Residual VQ         & 29.06 & 0.7798    & 84.56 & 83.67 \\ \hline
LookupFreeQuan.     & 19.09 & 0.8253    & 85.09 & 84.02 \\ \hline
BSQ                 & 20.70 & 0.7795    & 84.56 & 82.75 \\ \whline
\end{tabular}
}}
\label{tab:quan}
\end{minipage}%

\noindent Quantization (LFQ) achieve slightly higher Diff-SSIM scores of 0.8249 and 0.8253 respectively, FSQ provides the best overall generative fidelity with an rFID of 13.77. Binary Spherical Quantization (BSQ) reaches an rFID of 20.70 and an $F_1$ of 82.75, also below FSQ. This suggests that while standard VQ variants can capture structural changes, FSQ better preserves the semantic compressed representation required for high-fidelity MRI synthesis.

\paragraph{Which FSQ size achieves the best trade-off between generation and reconstruction?}

\begin{table*}[h]
\vspace{-7mm}
\centering
\caption{Ablation of FSQ grid configurations across reconstruction (left) and generative forecasting (right). The configuration $[8, 8, 8, 5, 5, 5]$ (highlighted) achieves the optimal trade-off between representational capacity and noise suppression, leading to superior clinical diagnostic performance.}
\vspace{-5mm}
\begin{minipage}{.49\textwidth}
\centering
\renewcommand{\arraystretch}{1.0}
\subcaption{Reconstruction Performance (Stage 1)}
\vspace{-3mm}
\resizebox{\textwidth}{!}{
\begin{tabular}{c|c|c|c|c}
\whline
\textbf{FSQ Size} & rFID  & Diff-SSIM & Acc.  & $F_1$  \\ \whline
{[}8,6,5{]}       & 6.93  & 0.8319    & 86.49 & 84.90 \\ \hline
{[}8,5,5,5{]}     & 14.95 & 0.8528    & 82.89 & 80.04 \\ \hline
{[}7,5,5,5,5{]}   & 6.36  & 0.8520    & 86.67 & 85.19 \\ \hline
{[}8,8,8,6,5{]}   & 24.24 & 0.8579    & 82.54 & 79.46 \\ \hline
\rowcolor[HTML]{EFEFEF} 
{[}8,8,8,5,5,5{]} & 9.45  & 0.8536    & 87.98 & 87.01 \\ \whline
\end{tabular}
}
\end{minipage}
\hfill
\begin{minipage}{.49\textwidth}
\centering
\subcaption{Generative Forecasting (Stage 2)}
\vspace{-3mm}
\renewcommand{\arraystretch}{1.0}
\resizebox{\textwidth}{!}{
\begin{tabular}{c|c|c|c|c}
\whline
\textbf{FSQ Size} & rFID  & Diff-SSIM & Acc.  & $F_1$  \\ \whline
{[}8,6,5{]}       & 7.07  & 0.8128    & 86.84 & 85.27 \\ \hline
{[}8,5,5,5{]}     & 15.58 & 0.8240    & 82.46 & 79.51 \\ \hline
{[}7,5,5,5,5{]}   & 7.00  & 0.8253    & 85.35 & 83.64 \\ \hline
{[}8,8,8,6,5{]}   & 27.07 & 0.8273    & 82.81 & 80.00 \\ \hline
\rowcolor[HTML]{EFEFEF} 
{[}8,8,8,5,5,5{]} & 12.78 & 0.8224    & 88.24 & 87.41 \\ \whline
\end{tabular}
}
\end{minipage}

\label{tab:fsq_size_ablation}
\vspace{-3mm}
\end{table*}

Table~\ref{tab:fsq_size_ablation} shows a trade-off: the FSQ step size must suppress high-frequency noise while preserving sparse biological drift. Coarse configurations (e.g., $[8, 6, 5]$) discard subtle indicators, while high-capacity grids (e.g., $[8, 8, 8, 6, 5]$) degrade stability (rFID 27.07) by interpolating imaging noise, which is the Continuous Interpolation Trap.

Configuration $[8, 8, 8, 5, 5, 5]$ balances expressive power with signal isolation through a non-Lipschitz threshold, reaching the best clinical metrics (Accuracy 88.24\%, $F_1$ 87.41\%). Other configurations yield higher structural similarity but lose diagnostic reliability by capturing non-semantic fluctuations. Calibrated quantization therefore matters for neuro-forecasting as more than a compression step.

\noindent 
\begin{minipage}[t]{0.48\linewidth} 
\vspace{0pt} 
    \paragraph{Which prediction target is better?}
  Table~\ref{tab:genmethod} compares generation paradigms and shows that a high-capacity, conditioned sequence model is needed to forecast the discrete progression signal. While the ContextCon. baseline reaches a competitive rFID of 14.49,
\end{minipage}
\hfill 
\begin{minipage}[t]{0.48\linewidth} 
\vspace{0pt}\vspace{\dimexpr-\baselineskip+0.3em\relax} 
\captionof{table}{Comparison of Generation Paradigms.}
\renewcommand{\arraystretch}{1.0}{
\resizebox{\textwidth}{!}{
\begin{tabular}{c|c|c|c|c}
\whline
\textbf{Generator} & rFID  & Diff-SSIM & Acc.  & $F_1$  \\ \whline
ContextCon.        & 14.49 & 0.7954    & 83.77 & 81.92 \\ \hline
MaskedCon.         & 23.17 & 0.8169    & 88.07 & 87.41 \\ \hline
\rowcolor[HTML]{EFEFEF}
Ours               & 13.92 & 0.8205    & 88.33 & 87.52 \\ \whline
\end{tabular}
}}
\label{tab:genmethod}
\end{minipage}
\vspace{2mm}

\noindent its relatively low Diff-SSIM of 0.7954 and Accuracy of 83.77\% suggest that simple conditioning on baseline anatomy is insufficient to capture the sparse pathological progression. Conversely, the MaskedCon. variant improves structural capturing, yet it suffers from a significant degradation in generative fidelity. This trade-off reflects the difficulty of maintaining global anatomical consistency while localized pathological changes are being synthesized.

%% file: Tables/Tab_Main.tex
\begin{table}[t]
\caption{\textbf{Quantitative comparison of neuro-forecasting performance on the longitudinal 3D brain MRI dataset.} We evaluate models across three dimensions: (i) \textbf{Generative Fidelity} via FID (i3d and classifier-based), (ii) \textbf{Structural Similarity} between forecasted and ground-truth progression, and (iii) \textbf{Downstream Clinical Utility} via AD diagnosis performance. Bold and underlined values indicate the best and second-best results, respectively.}
\label{tab:main}

\renewcommand{\arraystretch}{1.15}{
\resizebox{\textwidth}{!}{

\begin{tabular}{cccccccccc}
\whline
\multicolumn{1}{c|}{}                                                           & \multicolumn{2}{c|}{\textbf{Generative Fidelity}}                   & \multicolumn{3}{c|}{\textbf{Structural Similarity}}                                                              & \multicolumn{4}{c}{\textbf{Downstream Clinical Utility}}                                                                                                            \\ \cline{2-10} 
\multicolumn{1}{c|}{\multirow{-2}{*}{\textbf{Method}}}                          & FID (i3d)      & \multicolumn{1}{c|}{FID (cls.)}                    & Diff-SSIM       & Pix-SSIM        & \multicolumn{1}{c|}{NCC}                                                 & $ \mathbf{F_1}$                                              & $Pre.$           & $Rec.$           & \textbf{Acc.}  \\ \whline
\multicolumn{1}{c|}{\textit{CycleGAN}\cite{zhu2017unpaired}}                              & \textbf{12.13} & \multicolumn{1}{c|}{\textbf{7.03}}                 & 0.7713          & \textbf{0.9612} & \multicolumn{1}{c|}{0.9611}                                              & 75.44                                                        & 85.95          & 67.22          & 79.39    \\ \hline
\multicolumn{10}{c}{Diffusion based solutions}                                                                                                                                                                                                                                                                                                                                                                                                                    \\ \hline
\multicolumn{1}{c|}{\textit{Palette}\cite{saharia2022palette}}                               & 46.74          & \multicolumn{1}{c|}{9.27}                          & {0.8030}    & 0.8088          & \multicolumn{1}{c|}{{0.9810}}                                        & {80.36}                                                  & {\ul 87.50}    & {74.30}    & 82.89          \\ \hline
\multicolumn{1}{c|}{\textit{I2I-DiT}\cite{peebles2023scalable}}                               & 25.65          & \multicolumn{1}{c|}{{\ul 7.63}}                    & 0.7799          & 0.7979          & \multicolumn{1}{c|}{0.9807}                                              & \multicolumn{1}{c|}{62.37}                                   & 87.00          & 48.60          & 72.36          \\ \hline
\multicolumn{1}{c|}{\textit{BrLP}\cite{puglisi2024enhancing}}                               & 19.59          & \multicolumn{1}{c|}{{10.14}}                    & 0.8063          & 0.8574          & \multicolumn{1}{c|}{{\ul 0.9819}}                                              & \multicolumn{1}{c|}{80.36}                                   & 86.98          & 74.67         & 82.81         \\ \hline

\multicolumn{10}{c}{Auto-regressive based solutions}                                                                                                                                                                                                                                                                                                                                                                                                              \\ \hline
\multicolumn{1}{c|}{\textit{VQGAN}\cite{esser2021taming}}                                 & 30.11          & \multicolumn{1}{c|}{18.29}                         & 0.7625          & 0.8612          & \multicolumn{1}{c|}{0.9576}                                              & \multicolumn{1}{c|}{70.55}                                   & 78.23          & 64.24          & 74.74          \\ \hline
\multicolumn{1}{c|}{\textit{RQ-Transformer}\cite{lee2022autoregressive}}                        & 27.71          & \multicolumn{1}{c|}{10.89}                         & 0.7071          & 0.7507          & \multicolumn{1}{c|}{0.8666}                                              & \multicolumn{1}{c|}{60.78}                                   & 62.82          & 58.84          & 64.21          \\ \hline
\multicolumn{1}{c|}{\textit{NeuroAR}\cite{Yesiloglu2025NeuralAM}}                        &  22.28         & \multicolumn{1}{c|}{15.34}                         & {\ul 0.8128}          & 0.8554          & \multicolumn{1}{c|}{0.9782}                                              & \multicolumn{1}{c|}{{\ul 81.86}}                                   & {\ul 87.50}          & {\ul 76.91}          & {\ul 83.95}          \\ \hline
\rowcolor[HTML]{EFEFEF} 
\multicolumn{1}{c|}{\cellcolor[HTML]{EFEFEF}\textit{\textbf{Ours}}} & {\ul 13.92}    & \multicolumn{1}{c|}{\cellcolor[HTML]{EFEFEF}12.77} & \textbf{0.8204} & {\ul 0.9483}    & \multicolumn{1}{c|}{\cellcolor[HTML]{EFEFEF}\textbf{0.9880}} & \multicolumn{1}{c|}{\cellcolor[HTML]{EFEFEF}\textbf{87.51}} & \textbf{88.26} & \textbf{86.78} & \textbf{88.33} \\ \whline
\end{tabular}
}}

\end{table}

%% file: Sections/Related_Work.tex
\noindent \textbf{Static brain MRI generation.} The first generation of work established that deep generative models can produce anatomically realistic brain volumes. 3D GANs combining variational autoencoders with adversarial training were shown to synthesize whole-brain MRIs that capture the statistical distribution of healthy and pathological anatomy~\cite{Kwon2019GenerationO3, Zoghby2024GenerativeAN}. Conditional diffusion probabilistic models subsequently improved both fidelity and diversity, generating brain images whose distributional properties closely match those of real clinical scans~\cite{Pinaya2022BrainIG}. In parallel, cross-modality translation networks demonstrated that generative models can learn the highly nonlinear mapping between MRI contrasts~\cite{Yang2020MRICI, Sharma2019MissingMP}. These results show that modern generative architectures have enough representational capacity to model 3D neuroanatomy. However, all of the above operate on single time points: they produce a plausible brain image but encode no notion of how that anatomy evolves over time.

\noindent \textbf{Conditioned brain aging and disease simulation.} A second line of work introduces temporal and pathological conditioning to steer generation. Age-conditioned GANs learn to transform an input scan into what the same brain might look like at a different age by disentangling identity-preserving content from age- and sex-dependent style~\cite{Gadewar2023PredictingIB}. CounterSynth extends this idea to counterfactual reasoning: a conditional model of diffeomorphic deformations modifies select morphological features of a brain image to reflect a specified demographic or diagnostic label while leaving subject-specific details intact~\cite{Pombo2021EquitableMO}. Diffusion-based counterfactual generation further enables interpretable detection of Alzheimer's disease effects by comparing an individual's real scan with a model-generated healthy counterfactual of matched age and sex~\cite{Dhinagar2024CounterfactualMG}. Most recently, InBrainSyn combines cohort-level template learning with parallel transport to simulate individualized aging trajectories under both healthy and disease conditions from a single observation~\cite{Fu2025SynthesizingIA}. These methods demonstrate that generative models can be meaningfully conditioned on clinically relevant variables. 

\noindent \textbf{Individualized longitudinal prediction.} TimeFlow conditions a U-Net on continuous age variables under symmetric diffeomorphic constraints, enabling nonlinear extrapolation of individual atrophy trajectories~\cite{Jian2025TimeFlowTC}. Deformation-aware diffusion models incorporate explicit morphological priors into the generation process to better capture the characteristic atrophy patterns of Alzheimer's disease~\cite{Honga2025DeformationawareTG}. An alternative strategy bypasses the deformation assumption and directly generates future voxel intensities: conditional GANs with 3D discriminators have been trained to synthesize whole-brain MRIs at a target future time point~\cite{Jung2022ConditionalGW}. The autoregressive NeuroAR framework tokenizes 3D brain volumes and predicts future token maps via cross-attention on acquisition and target ages, outperforming both GAN and latent-diffusion baselines across multiple longitudinal cohorts~\cite{Yesiloglu2025NeuralAM}. Brain Latent Progression further operates in a compressed latent space to model individualized spatiotemporal disease trajectories via conditional latent diffusion~\cite{Puglisi2025BrainLP}.

%% file: Sections/Conclusion.tex
This work identifies and addresses two optimization pathologies, Identity Collapse and the Continuous Interpolation Trap, that cause standard generative models to fail when forecasting slow-evolving neurodegeneration. We introduce Latent Drift, a progressive framework that shifts the predictive target from absolute pixel-level anatomy to a compressed temporal residual. By applying Finite Scalar Quantization as a non-Lipschitz topological dead-zone filter, our approach suppresses dense nuisance noise while preserving the sparse biological progression. Experiments on longitudinal 3D brain MRIs show that Latent Drift improves forecasting fidelity and clinical diagnostic accuracy over diffusion and autoregressive baselines, providing a patient-specific forecasting tool that can support clinical trial design and earlier intervention. \\
\noindent\textbf{Limitations and Future Work.} Our shared dead-zone quantization grid assumes a stable background, so it can under-represent regions that change quickly within a clinical interval, such as the ventricles, while still filtering subtle change elsewhere. We also validate only on well-registered, downsampled ADNI and AIBL cohorts, leaving open how to separate subtle progression from large site-specific deformation (e.g., gradient nonlinearity) in raw multi-site data. Future work will address both through a region-aware quantization grid and validation on raw multi-site scans, and will extend the framework to other slow-evolving pathologies and modalities such as chest X-ray and retinal imaging.

%% file: Sections/proof_draft.tex
\subsection{Limits of Lipschitz Generative Models Under Dense Nuisance Noise}

\subsubsection{Problem Setup}

Let $(\mathcal X, \|\cdot\|_2)$ be a compact metric space and let 
$\Delta z \in \mathbb R^d$ denotes the latent temporal drift.
Assume the additive generative process
\begin{equation}
\Delta z = \delta(x) + \eta,
\end{equation}
where:

\begin{itemize}
    \item $\delta : \mathcal X \to \mathbb R^d$ is the true pathological progression signal, assumed to be $L_\delta$-Lipschitz continuous.
    There exists a support set $S(x) \subset \{1,\dots,d\}$ with
    $|S(x)| \le s \ll d$ such that
    \begin{equation}
        \delta_j(x) = 0 \quad \text{for all } j \notin S(x).
    \end{equation}
    
    Furthermore, on its support, the signal has high localized magnitude:
    \begin{equation}
        \min_{j \in S(x)} |\delta_j(x)| \ge \gamma > 0.
    \end{equation}

    \item $\eta \in \mathbb R^d$ is the sample-personalized nuisance noise drawn from a distribution $\mathcal{P}_\eta$ satisfying:
    \begin{enumerate}
        \item (Dense support) $\mathbb P(\eta_j \neq 0) = 1$ for all $j$,
        \item (Bounded continuous amplitude) $\|\eta\|_\infty \le \varepsilon$,
        \item (Zero mean) $\mathbb E[\eta] = 0$,
        \item (Independence) Independent of $x$ and across samples.
    \end{enumerate}
\end{itemize}

Let $\mathcal{D}_n = \{(x_i,y_i)\}_{i=1}^n$ be an i.i.d.\ dataset where $y_i = \delta(x_i) + \eta_i$.

We consider a hypothesis class of high-capacity continuous generative models:
\begin{equation}
    \mathcal F_L := 
\left\{
f:\mathcal X \to \mathbb R^d \;\middle|\;
\|f(x) - f(x')\|_2 \le L \|x-x'\|_2
\right\}.
\end{equation}

The model $f_\theta \in \mathcal F_L$ is trained via empirical risk minimization (ERM) on $\mathcal{D}_n$. We analyze the expected population signal error:
\begin{equation}
\mathcal R_\delta(f)
=
\mathbb E_{x \sim \mathcal{P_X}} \|f(x) - \delta(x)\|_2.
\end{equation}

\bigskip

\begin{theorem}[Irreducible Signal Error and Loss of Sparsity]
\label{thm:lipschitz_failure}
Assume $f_\theta \in \mathcal F_L$ achieves zero empirical risk (i.e., it interpolates the noisy training data, $f_\theta(x_i) = y_i$ for all $i \in \{1 \dots n\}$). Let $\rho(x) = \min_{x_i \in \mathcal{D}_n} \|x - x_i\|_2$ be the distance from a test point to its nearest training neighbor. Then:

\begin{enumerate}
    \item (Generalization Lower Bound)
    \begin{equation}
    \mathcal R_\delta(f_\theta)
    \ge
    \mathbb E\|\eta\|_2
    -
    (L + L_\delta) \mathbb E[\rho(x)].
    \end{equation}
    In particular, as $n \to \infty$ with a fixed Lipschitz constant $L$, $\mathbb E[\rho(x)] \to 0$, forcing the error on the true signal to be strictly lower-bounded by the noise amplitude.

    \item (Spurious Continuous Variance)
    For any test point $x_{\mathrm{test}}$ satisfying $\rho(x_{\mathrm{test}}) < \frac{|\eta_{i,j}|}{L}$, the model perfectly hallucinates the noise:
    \begin{equation}
        f_{\theta, j}(x_{\mathrm{test}}) \neq 0.
    \end{equation}
    
    Consequently, $f_\theta(x_{\mathrm{test}})$ is a dense vector on any bounded neighborhood of the training manifold, destroying the $s$-sparsity of $\delta$.
\end{enumerate}
\end{theorem}

\begin{proof}

\noindent\textbf{Part I: Lower Bound on Signal Error.}

Let $x \sim \mathcal{P_X}$ be an unseen test point, and let $x_{\pi(x)} \in \mathcal{D}_n$ be its nearest neighbor in the training set. By the triangle inequality:
\begin{equation}
    \|f_\theta(x_{\pi(x)}) - \delta(x_{\pi(x)})\|_2 
\le 
\|f_\theta(x_{\pi(x)}) - f_\theta(x)\|_2 
+ \|f_\theta(x) - \delta(x)\|_2 
+ \|\delta(x) - \delta(x_{\pi(x)})\|_2.
\end{equation}

Rearranging to isolate the term of interest:
\begin{equation}
    \|f_\theta(x) - \delta(x)\|_2 
\ge 
\|f_\theta(x_{\pi(x)}) - \delta(x_{\pi(x)})\|_2 
- \|f_\theta(x_{\pi(x)}) - f_\theta(x)\|_2 
- \|\delta(x) - \delta(x_{\pi(x)})\|_2.
\end{equation}

Because $f_\theta$ interpolates the training set, we have exactly $f_\theta(x_{\pi(x)}) = \delta(x_{\pi(x)}) + \eta_{\pi(x)}$. Therefore, the first term simplifies to $\|\eta_{\pi(x)}\|_2$. 

Applying the Lipschitz assumptions for $f_\theta$ and $\delta$, we bound the remaining terms:
\begin{equation}
    \|f_\theta(x) - \delta(x)\|_2 
\ge 
\|\eta_{\pi(x)}\|_2 
- L \|x - x_{\pi(x)}\|_2 
- L_\delta \|x - x_{\pi(x)}\|_2.
\end{equation}

Taking the expectation over $x$ and the noise distribution yields:
\begin{equation}
    \mathcal R_\delta(f_\theta) 
\ge 
\mathbb E\|\eta\|_2 - (L + L_\delta) \mathbb E[\rho(x)].
\end{equation}

This demonstrates that interpolating dense noise injects a permanent generalization bias proportional to the noise norm, which cannot be resolved by increasing dataset size $n$.

\noindent\textbf{Part II: Destruction of Sparsity.}

Fix a coordinate $j \notin S(x_{\mathrm{test}})$. By definition, the true signal is zero: $\delta_j(x_{\mathrm{test}}) = 0$. 

Let $x_i$ be the nearest training point to $x_{\mathrm{test}}$. On the training data, interpolation forces:
\begin{equation}
    f_{\theta,j}(x_i) = \delta_j(x_i) + \eta_{i,j}.
\end{equation}

Even if $\delta_j(x_i) = 0$, $f_{\theta,j}(x_i) = \eta_{i,j}$, which is non-zero with probability 1 (Assumption 1).

Because $f_\theta$ is $L$-Lipschitz continuous:
\begin{equation}
    |f_{\theta,j}(x_{\mathrm{test}}) - f_{\theta,j}(x_i)|
\le
L \|x_{\mathrm{test}} - x_i\|_2.
\end{equation}

Using the reverse triangle inequality, the predicted magnitude at the test point is strictly bounded away from zero:
\begin{equation}
    |f_{\theta,j}(x_{\mathrm{test}})|
\ge
|\eta_{i,j}|
-
L \|x_{\mathrm{test}} - x_i\|_2.
\end{equation}

Therefore, for $f_{\theta,j}(x_{\mathrm{test}})$ to be exactly zero, it strictly requires $\|x_{\mathrm{test}} - x_i\| \ge \frac{|\eta_{i,j}|}{L}$. For any test point falling within this local radius, the continuous model is mathematically forced to output a spurious non-zero value. Because $\eta_i$ is dense across all $d$ coordinates, the output $f_\theta(x_{\mathrm{test}})$ is correspondingly dense.

\end{proof}

\bigskip

\begin{corollary}[Necessity of Non-Lipschitz Projection]

Assume the noise amplitude is strictly bounded below the signal magnitude: $\varepsilon < \gamma$.
Define the coordinate-wise non-linear dead-zone operator $\mathcal Q: \mathbb{R}^d \to \mathbb{R}^d$ as:
\begin{equation}
    \mathcal Q(z)_j
=
\begin{cases}
z_j, & |z_j| > \varepsilon, \\
0, & |z_j| \le \varepsilon.
\end{cases}
\end{equation}

Because $|\eta_j| \le \varepsilon$ and $|\delta_j + \eta_j| \ge \gamma - \varepsilon > \varepsilon$ on the true support, applying this operator to the empirical targets yields perfect isolation of the biological signal:
\begin{equation}
    \mathcal Q(\delta(x) + \eta)
=
\delta(x)
\quad \text{for all } x.
\end{equation}

Crucially, $\mathcal Q$ introduces an infinite gradient at $|z_j| = \varepsilon$ and is strictly non-Lipschitz. 
Consequently, Theorem \ref{thm:lipschitz_failure} implies that exact recovery of the $s$-sparse biological progression $\delta$ from noisy targets $y$ is impossible using standard ERM over continuous hypothesis classes $\mathcal F_L$. It fundamentally requires projecting the learning objective through a non-Lipschitz operator to break the continuous interpolation trap.
\end{corollary}

\subsection{Formal Analysis of FSQ as a Support-Restoring Projection}

\subsubsection{Problem Setup}

Let $(\mathcal{X}, \|\cdot\|_2)$ be a compact metric space. 
For each $x \in \mathcal{X}$, the latent temporal drift satisfies

\begin{equation}
\Delta z(x) = \delta(x) + \eta,
\end{equation}

where:

\begin{itemize}
    \item $\delta : \mathcal{X} \to \mathbb{R}^d$ is the true pathological drift,
    \item $\eta \in \mathbb{R}^d$ is nuisance noise independent of $x$.
\end{itemize}

We assume:

\begin{enumerate}
    \item \textbf{Sparsity:} There exists a support set $S(x) \subset \{1,\dots,d\}$ 
    with $|S(x)| \le s \ll d$ such that
    \begin{equation}
        \delta_j(x) = 0 \quad \forall j \notin S(x).
    \end{equation}

    \item \textbf{Signal Strength:}
    \begin{equation}
        \min_{j \in S(x)} |\delta_j(x)| \ge \gamma > 0.
    \end{equation}

    \item \textbf{Dense Noise:}
    \begin{equation}
        \mathbb{P}(\eta_j \neq 0) = 1 \quad \forall j,
    \end{equation}
    and
    \begin{equation}
        \|\eta\|_\infty \le \varepsilon.
    \end{equation}

    \item \textbf{Separation Condition:}
    \begin{equation}
        0 < \varepsilon < \gamma.
    \end{equation}

\end{enumerate}

Define the scalar quantizer with step size $h > 0$:

\begin{equation}
Q_h(z)_j = h \cdot \left\lfloor \frac{z_j}{h} \right\rceil.
\end{equation}

We assume calibration:

\begin{equation}
h > 2\varepsilon.
\end{equation}

---

\subsubsection{Main Theorem}

\begin{theorem}[Exact Support Recovery via Quantization]
\label{thm:fsq_support_recovery}
Under the assumptions above, define
\begin{equation}
    \Delta z_q(x) := Q_h(\delta(x) + \eta).
\end{equation}
Then for all $x \in \mathcal{X}$:
\begin{enumerate}
    \item (Annihilation outside support)
    \begin{equation}
        \Delta z_{q,j}(x) = 0
    \quad \forall j \notin S(x).
    \end{equation}

    \item (Support preservation)
    \begin{equation}
        \mathrm{supp}(\Delta z_q(x)) = S(x).
    \end{equation}

    \item (Sparsity restoration)
    \begin{equation}
        \|\Delta z_q(x)\|_0 \le s.
    \end{equation}
    
\end{enumerate}
\end{theorem}

---

\subsubsection{Proof}

Fix $x \in \mathcal{X}$.

\paragraph{Step 1: Coordinates outside the true support.}

For any $j \notin S(x)$, we have $\delta_j(x) = 0$, hence

\begin{equation}
    \Delta z_j(x) = \eta_j.
\end{equation}

By boundedness of noise,

\begin{equation}
    |\Delta z_j(x)| = |\eta_j| \le \varepsilon.
\end{equation}

Because $h > 2\varepsilon$, we have

\begin{equation}
    |\eta_j| < \frac{h}{2}.
\end{equation}

Therefore rounding maps the coordinate to zero:

\begin{equation}
    Q_h(\Delta z_j(x)) = 0.
\end{equation}

This proves annihilation outside the support.

---

\paragraph{Step 2: Coordinates on the true support.}

Let $j \in S(x)$. Then

\begin{equation}
    \Delta z_j(x) = \delta_j(x) + \eta_j.
\end{equation}

Using triangle inequality,

\begin{equation}
    |\Delta z_j(x)|
\ge |\delta_j(x)| - |\eta_j|
\ge \gamma - \varepsilon.
\end{equation}

Since $\varepsilon < \gamma$, we have $\gamma - \varepsilon > 0$.

Because $h > 2\varepsilon$, we have

\begin{equation}
    \gamma - \varepsilon > \frac{h}{2}
\quad \text{(by choosing $h < 2\gamma$ if desired)}.
\end{equation}

Thus $|\Delta z_j(x)| \ge h/2$, meaning it cannot fall into the zero bin.

Therefore

\begin{equation}
    Q_h(\Delta z_j(x)) \neq 0.
\end{equation}

This proves support preservation.

---

\paragraph{Step 3: Sparsity restoration.}

From Steps 1 and 2,

\begin{equation}
    \mathrm{supp}(\Delta z_q(x)) = S(x).
\end{equation}

Thus

\begin{equation}
    \|\Delta z_q(x)\|_0 = |S(x)| \le s.
\end{equation}

\hfill $\square$

---

\subsubsection{Necessity of Non-Lipschitz Projection}

We now formalize why this mechanism cannot be achieved by any Lipschitz map.

\begin{proposition}[No Lipschitz Map Can Exactly Restore Support]
Let $F : \mathbb{R}^d \to \mathbb{R}^d$ be globally $L$-Lipschitz. 
Assume $\eta$ is dense with $\mathbb{P}(\eta_j \neq 0)=1$.
Then no such $F$ can satisfy

\begin{equation}
    F(\delta(x)+\eta)_j = 0 \quad \forall j \notin S(x),
\end{equation}

for all realizations of $\eta$.
\end{proposition}

\begin{proof}

Suppose such $F$ exists.

Fix $j \notin S(x)$.
For training points, $\delta_j(x)=0$, so the input varies over the interval $[-\varepsilon,\varepsilon]$ due to $\eta_j$.

Because $F$ is Lipschitz, its restriction to this interval is continuous.

If $F$ maps the entire interval to zero, then its derivative must vanish on that interval. 
But Lipschitz continuity implies small perturbations in neighboring coordinates propagate continuously. 

Since $\eta$ is dense across all coordinates, arbitrarily small perturbations in input produce small but non-zero perturbations in output.

Therefore exact annihilation over an open interval contradicts global Lipschitz continuity unless $F$ is locally constant on a full neighborhood, which would also annihilate signal coordinates.

Thus no globally Lipschitz map can exactly restore sparse support under dense perturbations.
\end{proof}

\subsubsection{Conclusion}

The quantizer $Q_h$ introduces a discontinuity at $|z_j| = h/2$, making it non-Lipschitz. 
This discontinuity is precisely what enables exact support restoration.

Therefore:

\begin{itemize}
    \item Continuous predictors inevitably smear dense noise (Theorem 3).
    \item Exact sparse recovery requires a non-Lipschitz projection.
    \item FSQ provides this projection via a calibrated dead-zone.
\end{itemize}

This formally justifies the necessity of quantized latent drift for generative forecasting under dense nuisance perturbations.

%% file: Sections/Experiments_Appd.tex
\subsection{Benchmark Details}
\subsubsection{Data Preprocessing}
All images undergo standard preprocessing procedures, including bias field correction, skull stripping, and spatial normalization to the MNI152 template space\cite{fonov2011unbiased}.
The images are resampled to an isotropic resolution of 1 mm and organized into volumes of size $182 \times 218 \times 182$ for subsequent analysis.
The intensities are then converted to Hounsfield Units (HU) using the slope and intercept values obtained from the metadata and clipped to the range [-1000, 1000] HU, corresponding to the practical limits of the HU scale~\cite{denotter2019hounsfield,lamba2014ct}. Finally, the volumes are downsampled to a spatial size of $93 \times 112 \times 96$.

\subsubsection{Data Statistics}
Based on the available follow-up information, we select scan pairs with a time horizon $\Delta t$ no greater than 48 months.
A total of 1,195 patients are selected, forming 3,981 current-future scan pairs.
Among them, 342 patients with 1,140 scan pairs are used as the test set, while the remaining data are used for training and validation.

\subsection{Quantitative evaluation}

We evaluate the generated 3D MRIs along three dimensions:
\begin{itemize}
    \item[-] \textbf{Structural Similarity} measures spatial agreement with the real volumes. Pix-SSIM is the SSIM between the generated future MRI $\hat{x}_{fut}$ and the ground-truth (GT) $x_{fut}$, while Diff-SSIM is the SSIM on the difference maps relative to the current MRI $x_{cur}$, which isolates anatomical change over time. We also report Normalized Cross-Correlation (NCC) on the central sub-volume of $x \in \mathbb{R}^{H \times W \times D}$ with depth $D=40$, which focuses on the main brain regions and avoids boundary slices with little anatomical content.

    \item[-] \textbf{Downstream Clinical Utility} measures diagnostic value. We train an Alzheimer's Disease (AD) classifier on our preprocessed data, reaching over 91\% accuracy with a transformer architecture similar to ViViT\cite{arnab2021vivit}, then apply it to the MRIs generated by each method.

    \item[-] \textbf{Generative Fidelity} measures distribution-level realism via the Fréchet Inception Distance (FID). FID(i3d) uses features from a standard I3D model, and FID(cls.) uses features from the AD classifier above for a more domain-relevant comparison.
\end{itemize}
\subsection{Implementation Details} 
Training is conducted on $8 \times$ NVIDIA H200 GPUs using bf16 mixed-precision.
The Latent Drift Tokenizer is optimized with Adam ($\beta=(0.5,0.9)$), batch size 8, learning rate $1\times10^{-4}$ with cosine decay, and is trained for 920 epochs (80,040 steps) with 2 warmup epochs.
The Generator is optimized with AdamW ($\beta=(0.9,0.95)$), batch size 2, learning rate $3\times10^{-4}$ with cosine decay, and is trained for 70 epochs (24,360 steps) with 6 warmup epochs.
During inference, we use a sampling temperature of 1.2.

\subsection{Pseudo Code}
We give pseudo-code for both stages. Algorithm 1 covers the Latent Drift Tokenizer optimization, and Algorithm 2 the training and inference of the autoregressive generator.
                                                                
\begin{figure}[H]
\begin{minipage}[t]{0.48\textwidth}
\begin{algorithm}[H]
\caption{Latent Drift Tokenizer (Stage 1)}
\label{alg:tokenizer}
\begin{algorithmic}[1]
\STATE \texttt{{\color{codegreen}\# T: tokenizer (PatchEmbed + (2+1)D Transformer)}}
\STATE \texttt{{\color{codegreen}\# Q: FSQ quantizer}}
\STATE \texttt{{\color{codegreen}\# T\_inv: de-tokenizer (cross-attn decoder)}}
\STATE \texttt{{\color{codegreen}\# D: discriminator}}
\STATE
\STATE \texttt{for (x\_cur, x\_fut) in loader:}
\STATE \quad \texttt{{\color{codegreen}\# tokenize paired scans}}
\STATE \quad \texttt{p\_cur = PatchEmbed(x\_cur)}
\STATE \quad \texttt{p\_fut = PatchEmbed(x\_fut)}
\STATE \quad \texttt{z = SpatialAttn(cat(p\_cur, p\_fut))}
\STATE \quad \texttt{z = TemporalAttn(z)}
\STATE \quad \texttt{z\_cur, z\_fut = split(z)}
\STATE
\STATE \quad \texttt{{\color{codegreen}\# drifting module + quantization}}
\STATE \quad \texttt{dz = z\_fut - z\_cur}
\STATE \quad \texttt{dz\_q, idx = Q(dz)}
\STATE
\STATE \quad \texttt{{\color{codegreen}\# de-tokenize}}
\STATE \quad \texttt{x\_hat = T\_inv(z\_cur, dz\_q)}
\STATE
\STATE \quad \texttt{{\color{codegreen}\# generator update}}
\STATE \quad \texttt{L\_G = MSE(x\_fut, x\_hat)}
\STATE \quad \texttt{L\_G += LPIPS(x\_fut, x\_hat)}
\STATE \quad \texttt{L\_G += GAN\_loss(D(x\_hat))}
\STATE \quad \texttt{L\_G.backward()}
\STATE \quad \texttt{update(T, Q, T\_inv)}
\STATE
\STATE \quad \texttt{{\color{codegreen}\# discriminator update}}
\STATE \quad \texttt{L\_D = disc\_loss(D(x\_fut),}
\STATE \quad \texttt{\quad \quad D(x\_hat.detach()))}
\STATE \quad \texttt{L\_D.backward(); update(D)}
\end{algorithmic}
\end{algorithm}
\end{minipage}
\hfill
\begin{minipage}[t]{0.48\textwidth}
\begin{algorithm}[H]
\caption{Generator (Stage 2)}
\label{alg:generator}
\begin{algorithmic}[1]
\STATE \texttt{{\color{codegreen}\# G: auto-regressive transformer}}
\STATE \texttt{{\color{codegreen}\# T, Q, T\_inv: frozen from Stage 1}}
\STATE
\STATE \texttt{{\color{codegreen}\# --- Training ---}}
\STATE \texttt{for (x\_cur, x\_fut, c, dt) in loader:}
\STATE \quad \texttt{{\color{codegreen}\# ground-truth drift tokens}}
\STATE \quad \texttt{with no\_grad():}
\STATE \quad \quad \texttt{z\_cur, dz\_q\_gt = T\_Q(x\_cur, x\_fut)}
\STATE
\STATE \quad \texttt{{\color{codegreen}\# predict drift tokens}}
\STATE \quad \texttt{logits = G(z\_cur, c, dt)}
\STATE \quad \texttt{loss = CE(logits, dz\_q\_gt)}
\STATE \quad \texttt{loss.backward(); update(G)}
\STATE
\STATE \texttt{{\color{codegreen}\# --- Inference ---}}
\STATE \texttt{with no\_grad():}
\STATE \quad \texttt{{\color{codegreen}\# frozen tokenizer}}
\STATE \quad \texttt{z\_cur = T(x\_cur)}
\STATE
\STATE \quad \texttt{{\color{codegreen}\# sample drift tokens}}
\STATE \quad \texttt{logits = G(z\_cur, c, dt)}
\STATE \quad \texttt{dz\_q = sample(logits / tau)}
\STATE
\STATE \quad \texttt{{\color{codegreen}\# frozen de-tokenizer}}
\STATE \quad \texttt{x\_fut\_hat = T\_inv(z\_cur, dz\_q)}
\STATE \quad \texttt{return x\_fut\_hat}
\end{algorithmic}
\end{algorithm}
\end{minipage}
\end{figure}

%% file: Sections/Ablation_Appd.tex
In this section, we present the comprehensive results of our ablation studies, evaluating the performance variations across all metrics including generative fidelity, structural similarity, and downstream clinical utility.

\clearpage
\begingroup
\setlength{\intextsep}{6pt plus 30pt minus 4pt}

\begin{table}[H]
\caption{\textbf{Ablation of Prediction Targets}}
\label{tab:main}

\renewcommand{\arraystretch}{1.0}{
\resizebox{\textwidth}{!}{

\begin{tabular}{cccccccccc}
\whline
\multicolumn{1}{c|}{}                                                           & \multicolumn{2}{c|}{\textbf{Generative Fidelity}}                   & \multicolumn{3}{c|}{\textbf{Structural Similarity}}                                                              & \multicolumn{4}{c}{\textbf{Downstream Clinical Utility}}                                                                                                            \\ \cline{2-10} 
\multicolumn{1}{c|}{\multirow{-2}{*}{\textbf{Target}}}                          & FID (i3d)      & \multicolumn{1}{c|}{FID (cls.)}                    & Diff-SSIM       & Pix-SSIM        & \multicolumn{1}{c|}{NCC}                                                 & $ \mathbf{F_1}$                                              & $Pre.$           & $Rec.$           & \textbf{Acc.}  \\ \whline
\multicolumn{1}{c|}{Pixel}                        & 44.38          & \multicolumn{1}{c|}{7.04}                         & 0.7625          & 0.8730          & \multicolumn{1}{c|}{0.9576}                                              & \multicolumn{1}{c}{72.75}                                   & 83.21          & 64.62          & 77.19          \\ \hline
\rowcolor[HTML]{EFEFEF} 
\multicolumn{1}{c|}{\cellcolor[HTML]{EFEFEF} Drift (ours)} & {13.92}    & \multicolumn{1}{c|}{\cellcolor[HTML]{EFEFEF}12.77} & 0.8205 & { 0.9483}    & \multicolumn{1}{c|}{\cellcolor[HTML]{EFEFEF} 0.9880} & \multicolumn{1}{c}{\cellcolor[HTML]{EFEFEF}87.52} & 88.26 & 86.78 & 88.33 \\ \whline
\end{tabular}
}}

\end{table}
\begin{table}[H]
\caption{\textbf{Ablation of Quantization Mechanisms across Reconstruction}}
\label{tab:quant_methods}

\renewcommand{\arraystretch}{1.0}
\resizebox{\textwidth}{!}{
\begin{tabular}{c|cc|ccc|cccc}
\whline
\multirow{2}{*}{\textbf{Quantizers}} 
& \multicolumn{2}{c|}{\textbf{Generative Fidelity}} 
& \multicolumn{3}{c|}{\textbf{Structural Similarity}} 
& \multicolumn{4}{c}{\textbf{Downstream Clinical Utility}} \\ \cline{2-10}

& FID (i3d) & FID (cls.) 
& Diff-SSIM & Pix-SSIM & NCC 
& $\mathbf{F_1}$  & $Pre.$   & $Rec.$   & \textbf{Acc.} \\ \whline

\rowcolor[HTML]{EFEFEF}
FSQ (ours) 
& 11.11 & 9.46 
& 0.8536 & 0.9606 & 0.9925 
& 87.01 & 88.61 & 85.47 & 87.98 \\ \hline

VectorQuant. 
& 13.60 & 30.70 
& 0.8522 & 0.9606 & 0.9926 
& 85.52 & 87.52 & 83.61 & 86.67 \\ \hline

Residual VQ 
& 12.71 & 8.32 
& 0.8356 & 0.9434 & 0.9898 
& 84.87 & 86.77 & 83.05 & 86.05 \\ \hline

LookupFreeQuan. 
& 19.15 & 28.06 
& 0.8253 & 0.9523 & 0.9897 
& 84.02 & 84.82 & 83.24 & 85.09 \\ \whline

\end{tabular}
}
\end{table}
\begin{table}[H]
\caption{\textbf{Ablation of Quantization Mechanisms across Generative Forecasting}}
\label{tab:quant_methods}

\renewcommand{\arraystretch}{1.0}
\resizebox{\textwidth}{!}{
\begin{tabular}{c|cc|ccc|cccc}
\whline
\multirow{2}{*}{\textbf{Quantizers}} 
& \multicolumn{2}{c|}{\textbf{Generative Fidelity}} 
& \multicolumn{3}{c|}{\textbf{Structural Similarity}} 
& \multicolumn{4}{c}{\textbf{Downstream Clinical Utility}} \\ \cline{2-10}

& FID (i3d) & FID (cls.) 
& Diff-SSIM & Pix-SSIM & NCC 
& $\mathbf{F_1}$  & $Pre.$   & $Rec.$   & \textbf{Acc.} \\ \whline

\rowcolor[HTML]{EFEFEF}
FSQ (ours)
& 13.77 & 12.78
& 0.8205 & 0.9484 & 0.9897
& 87.51 & 88.26 & 86.78 & 88.33 \\ \hline

VectorQuant.
& 14.39 & 33.21
& 0.8252 & 0.9534 & 0.9901
& 84.89 & 87.23 & 82.68 & 86.14 \\ \hline

Residual VQ
& 290.67 & 75.75
& 0.7798 & 0.8716 & 0.9735
& 83.67 & 83.36 & 83.99 & 84.56 \\ \hline

LookupFreeQuan.
& 19.10 & 28.32
& 0.8254 & 0.9522 & 0.9897
& 83.91 & 84.79 & 83.05 & 85.00 \\ \whline

\end{tabular}
}
\end{table}
\begin{table}[H]
\caption{\textbf{Ablation of FSQ grid configurations across reconstruction}}
\label{tab:codebook_config}

\renewcommand{\arraystretch}{1.15}
\resizebox{\textwidth}{!}{
\begin{tabular}{c|cc|ccc|cccc}
\whline
\multirow{2}{*}{\textbf{FSQ Size}} 
& \multicolumn{2}{c|}{\textbf{Generative Fidelity}} 
& \multicolumn{3}{c|}{\textbf{Structural Similarity}} 
& \multicolumn{4}{c}{\textbf{Downstream Clinical Utility}} \\ \cline{2-10}

& FID (i3d) & FID (cls.) 
& Diff-SSIM & Pix-SSIM & NCC 
& $\mathbf{F_1}$  & $Pre.$   & $Rec.$   & \textbf{Acc.} \\ \whline

[8,6,5]
& 21.55 & 6.93
& 0.8319 & 0.9506 & 0.9916
& 84.90 & 89.65 & 80.63 & 86.49 \\ \hline

[8,5,5,5]
& 12.34 & 14.95
& 0.8528 & 0.9620 & 0.9924
& 80.04 & 88.86 & 72.81 & 82.89 \\ \hline

[7,5,5,5,5]
& 10.80 & 6.36
& 0.8520 & 0.9609 & 0.9921
& 85.19 & 89.37 & 81.38 & 86.67 \\ \hline

[8,8,8,6,5]
& 10.13 & 24.25
& 0.8579 & 0.9609 & 0.9924
& 79.46 & 89.12 & 71.69 & 82.54 \\ \hline

\rowcolor[HTML]{EFEFEF}
{[8,8,8,5,5,5]}
& 11.11 & 9.46
& 0.8536 & 0.9606 & 0.9925
& 87.01 & 88.61 & 85.47 & 87.98 \\ \whline

\end{tabular}
}
\end{table}
\begin{table}[H]
\caption{\textbf{Ablation of FSQ grid configurations across generative forecasting}}
\label{tab:codebook_config_forecasting}

\renewcommand{\arraystretch}{1.15}
\resizebox{\textwidth}{!}{
\begin{tabular}{c|ccccc|cccc}
\whline
\multirow{2}{*}{\textbf{FSQ Size}} 
& \multicolumn{5}{c|}{\textbf{Generative Fidelity \& Structural Similarity}} 
& \multicolumn{4}{c}{\textbf{Downstream Clinical Utility}} \\ \cline{2-10}

& FID (i3d) & FID (cls.) 
& Diff-SSIM & Pix-SSIM & NCC 
& $\mathbf{F_1}$  & $Pre.$   & $Rec.$   & \textbf{Acc.} \\ \whline

[8,6,5]
& 21.36 & 7.07
& 0.8128 & 0.9446 & 0.9894
& 85.27 & 90.23 & 80.82 & 86.84 \\ \hline

[8,5,5,5]
& 13.25 & 15.58
& 0.8241 & 0.9542 & 0.9894
& 79.51 & 88.38 & 72.25 & 82.46 \\ \hline

[7,5,5,5,5]
& 11.20 & 7.00
& 0.8254 & 0.9527 & 0.9891
& 83.64 & 88.22 & 79.52 & 85.35 \\ \hline

[8,8,8,6,5]
& 12.22 & 27.08
& 0.8273 & 0.9481 & 0.9880
& 80.00 & 88.49 & 73.00 & 82.81 \\ \hline

\rowcolor[HTML]{EFEFEF}
{[8,8,8,5,5,5]}
& 13.92 & 12.78
& 0.8225 & 0.9484 & 0.9897
& 87.41 & 88.24 & 86.59 & 88.25 \\ \whline

\end{tabular}
}
\end{table}
\endgroup
\pagebreak[4]

\begin{table}[!htbp]
\caption{\textbf{Ablation of Generator Architecture.}}
\label{tab:generator_architecture}

\renewcommand{\arraystretch}{1.15}
\resizebox{\textwidth}{!}{
\begin{tabular}{c|cc|ccc|cccc}
\whline
\multirow{2}{*}{\textbf{Generator}} 
& \multicolumn{2}{c|}{\textbf{Generative Fidelity}} 
& \multicolumn{3}{c|}{\textbf{Structural Similarity}} 
& \multicolumn{4}{c}{\textbf{Downstream Clinical Utility}} \\ \cline{2-10}

& FID (i3d) & FID (cls.) 
& Diff-SSIM & Pix-SSIM & NCC 
& $\mathbf{F_1}$  & $Pre.$   & $Rec.$   & \textbf{Acc.} \\ \whline

FSQ+DiT
& 18.45 & 21.26
& 0.7321 & 0.8300 & 0.9650
& 76.84 & 84.00 & 70.80 & 80.18 \\ \hline

\rowcolor[HTML]{EFEFEF}
Ours
& 13.92 & 12.78
& 0.8205 & 0.9484 & 0.9897
& 87.51 & 88.26 & 86.78 & 88.33 \\ \whline

\end{tabular}
}
\end{table}

We further examine the choice of generator family by replacing our autoregressive model with a diffusion transformer over the same quantized drift tokens (FSQ+DiT). As shown in Table~\ref{tab:generator_architecture}, FSQ+DiT degrades markedly across all dimensions (e.g., Diff-SSIM 0.7321, $F_1$ 76.84). The discrete dead-zone structure produced by FSQ breaks the smooth latent distribution that diffusion denoising relies on, so the autoregressive formulation is a more natural fit for modeling the quantized drift.

%% file: Sections/Vis_Appd.tex
Here we show more visualization of gradient-based signal differences between current and predicted future states in Fig.~\ref{fig:other_gradient}.

\begin{figure}[!thbp]
\centering

\begin{subfigure}{0.48\linewidth}
    \centering
    \includegraphics[width=\linewidth]{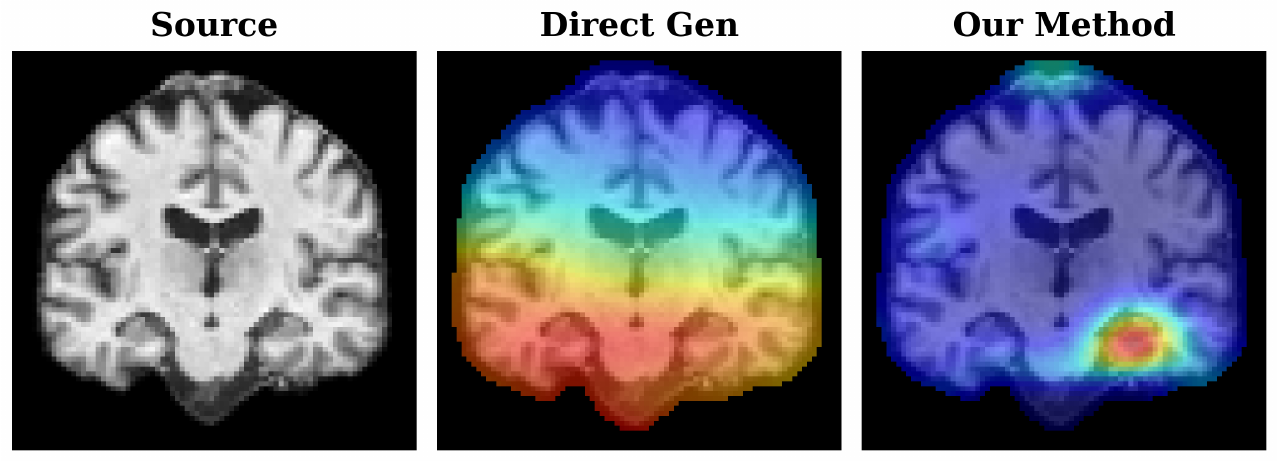}
    \caption{}
\end{subfigure}
\hfill
\begin{subfigure}{0.48\linewidth}
    \centering
    \includegraphics[width=\linewidth]{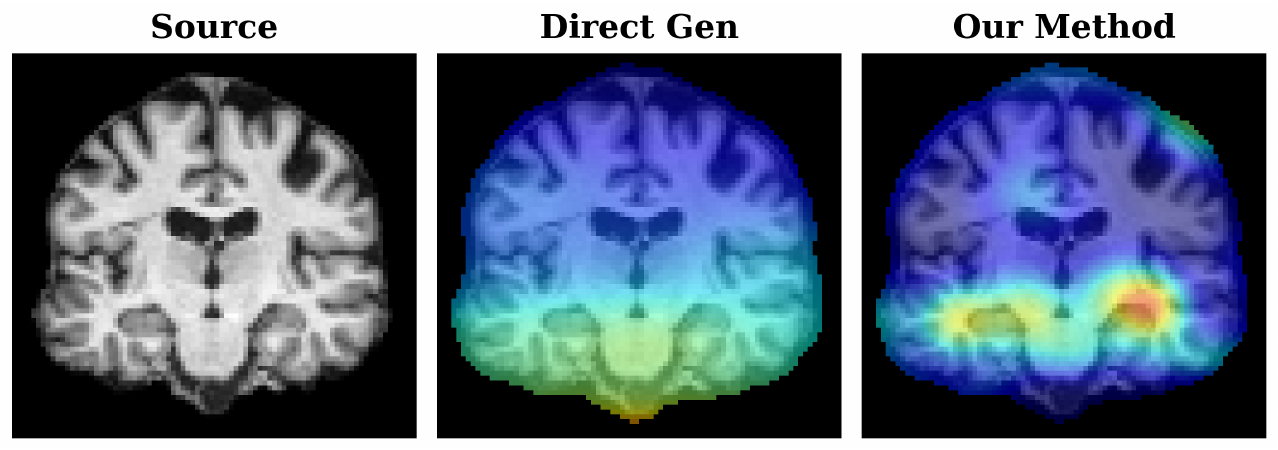}
    \caption{}
\end{subfigure}

\vspace{0.4cm}

\begin{subfigure}{0.48\linewidth}
    \centering
    \includegraphics[width=\linewidth]{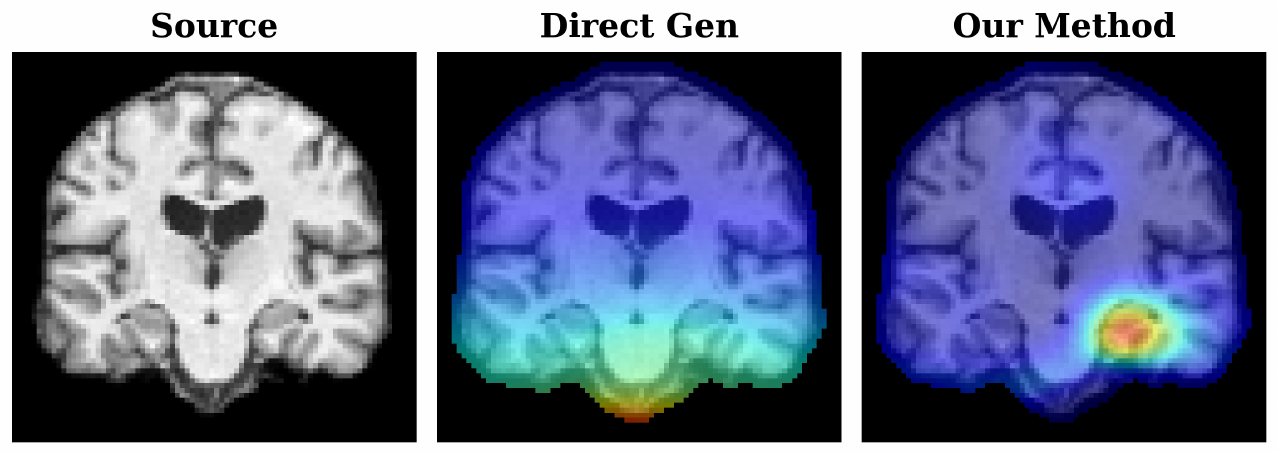}
    \caption{}
\end{subfigure}
\hfill
\begin{subfigure}{0.48\linewidth}
    \centering
    \includegraphics[width=\linewidth]{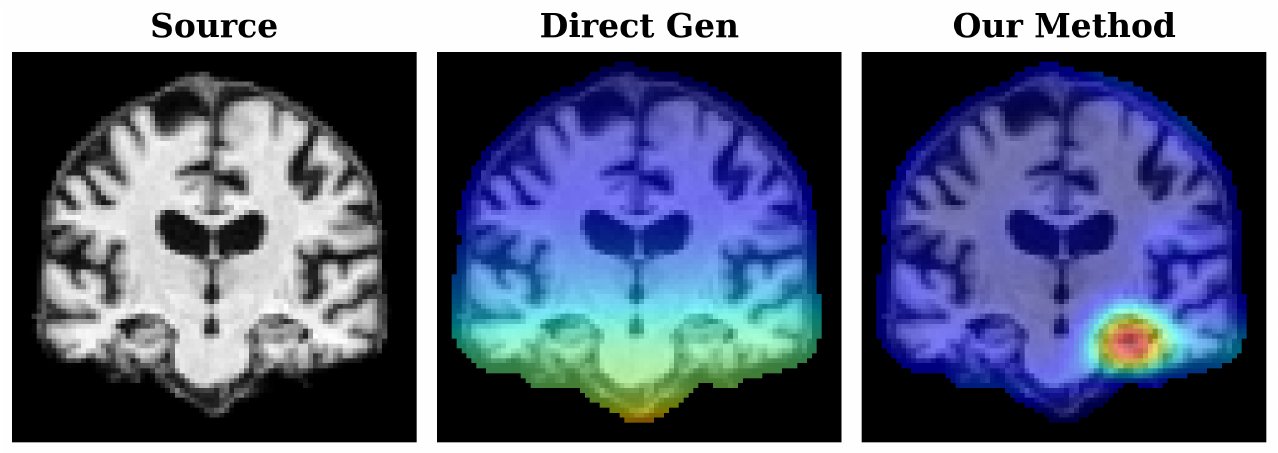}
    \caption{}
\end{subfigure}

\caption{Visualization of gradient-based signal differences}
\label{fig:other_gradient}
\end{figure}

\subsection{Qualitative Comparison}
Because progression over a typical clinical interval is extremely subtle, predictions from competing methods often look almost identical to the input scan, and differences are difficult to perceive in the raw volumes. Fig.~\ref{fig:qual_compare} compares our forecasts against baseline methods. Our method recovers the localized structural change that aligns with the ground-truth future scan, whereas the baselines largely reproduce the baseline anatomy. 

\begin{figure}[!htbp]
    \centering
    \includegraphics[width=\linewidth]{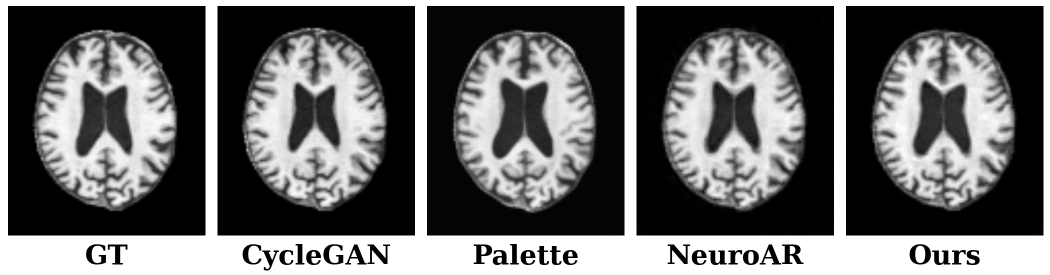}
    \caption{Qualitative comparison of forecasted future scans across methods. Our method captures localized progression consistent with the ground truth, while competing methods stay close to the input anatomy.}
    \label{fig:qual_compare}
\end{figure}